\documentclass[11pt]{article}
\usepackage[round]{natbib}
\ifdefined\usebigfont

\usepackage{times}
\usepackage[fontsize=13pt]{scrextend}
\usepackage[left=1.56in,right=1.56in,top=1.74in,bottom=1.74in]{geometry}
\else
\usepackage{fullpage}   
\usepackage{times}
\fi
\usepackage{makecell}

\usepackage{enumerate}
\usepackage{times}
\usepackage[OT1]{fontenc}
\usepackage{amsmath,amssymb}
\usepackage[usenames]{color}
\usepackage{multirow}
\usepackage[colorlinks,
            linkcolor=red,
            anchorcolor=blue,
            citecolor=blue
            ]{hyperref}

\usepackage[utf8]{inputenc} 
\usepackage[T1]{fontenc}    
\usepackage{booktabs}       
\usepackage{amsfonts}       
\usepackage{nicefrac}       
\usepackage{microtype}      

\usepackage{amsmath, amssymb}
\usepackage{algorithm}
\usepackage{algorithmic}
\usepackage{amsthm}
\usepackage[small]{caption}
\usepackage{graphicx}
\usepackage{subcaption}

\usepackage{enumitem}
\usepackage[olditem,oldenum]{paralist}

\newcommand{\calS}{\mathcal{S}}
\newcommand{\calA}{\mathcal{A}}

\newcommand{\calX}{\mathcal{X}}
\newcommand{\calF}{\mathcal{F}}
\newcommand{\calG}{\mathcal{G}}
\newcommand{\calO}{\mathcal{O}}

\newcommand{\R}{\mathbb{R}}
\renewcommand{\P}{\mathbb{P}}
\newcommand{\E}{\mathbb{E}}
\newcommand{\N}{\mathbb{N}}
\newcommand{\I}{\mathbf{I}}

\let\oldPr\Pr
\renewcommand{\Pr}{\oldPr\nolimits}

\newcommand{\vect}[1]{\ensuremath{\mathbf{#1}}}
\newcommand{\mat}[1]{\ensuremath{\mathbf{#1}}}

\newcommand{\w}{\vect{w}}

\newcommand{\A}{\mat{A}}

\newtheorem{assumption}{Assumption}
\newtheorem{proposition}{Proposition}
\newtheorem{lemma}{Lemma}
\newtheorem{theorem}{Theorem}
\newtheorem{definition}{Definition}
\newtheorem{corollary}{Corollary}

\newcommand{\norm}[1]{\left\|{#1}\right\|}
\newcommand{\fnorm}[1]{\left\|{#1}\right\|_{\text{F}}}

\usepackage{bm}
\newcommand{\bphi}{\bm{\phi}}
\newcommand{\bpsi}{\bm{\psi}}
\newcommand{\bmu}{\bm{\mu}}
\newcommand{\btheta}{\bm{\theta}}

\newcommand{\tlSigma}{\widetilde{\Sigma}}

\newcommand{\la}{\langle}
\newcommand{\ra}{\rangle}

\usepackage{scalerel,amssymb}

\newcommand{\argmax}{\mathop{\rm argmax}}

\usepackage[olditem,oldenum]{paralist}
\usepackage{tikz}
\usetikzlibrary{shapes.geometric}
\usepackage{mathrsfs}
\makeatletter
\newcommand*\bigcdot{\mathpalette\bigcdot@{.6}}
\newcommand*\bigcdot@[2]{\mathbin{\vcenter{\hbox{\scalebox{#2}{$\m@th#1\bullet$}}}}}
\makeatother

\def\N{{\mathbb{N}}}

\def\R{{\mathbb{R}}}

\renewcommand{\Pr}{\oldPr\nolimits}

\usepackage{hyperref}       

\hypersetup{
	plainpages=false,
	colorlinks=true,              
	linkcolor=blue,               
	anchorcolor=blue,             
	citecolor=blue,               
	filecolor=blue,               
	pagecolor=blue,               
	urlcolor=blue,                
	pdfview=FitH,                 
	pdfstartview=FitH,            
	pdfpagelayout=SinglePage      
}

\DeclareSymbolFont{extraup}{U}{zavm}{m}{n}
\DeclareMathSymbol{\varheart}{\mathalpha}{extraup}{86}
\DeclareMathSymbol{\vardiamond}{\mathalpha}{extraup}{87}

\title{\huge Efficient Learning in Non-Stationary Linear
Markov Decision Processes}
\author{Ahmed Touati ${}^\varheart$ \qquad
Pascal Vincent ${}^\varheart \,{}^\vardiamond \,{}^\clubsuit$
  }
\date{}

\begin{document}
\maketitle
\begin{abstract}
\let\thefootnote\relax\footnotetext{
$^\varheart$ Mila, Universit\'e de Montr\'eal,
$^\clubsuit$ Facebook AI Research,
$^\vardiamond$ Canada CIFAR AI Chair and CIFAR Associate Fellow
}

We study episodic reinforcement learning in non-stationary linear (a.k.a. low-rank) Markov Decision Processes (MDPs), \textit{i.e}, both the reward and transition kernel are linear with respect to a given feature map and are allowed to evolve either slowly or abruptly over time. For this
problem setting, we propose \textsc{opt-wlsvi} an optimistic model-free algorithm based on weighted least squares value iteration which uses exponential weights to smoothly forget data that are far in the past. We show that our algorithm, when competing against the best policy at each time, achieves a regret that is upper bounded by $\widetilde{\mathcal{O}}(d^{5/4}H^2 \Delta^{1/4} K^{3/4})$ where $d$ is the dimension of the feature space, $H$ is the planning horizon, $K$ is the number of episodes and $\Delta$ is a suitable measure of non-stationarity of the MDP. Moreover, we point out technical gaps in the study of forgetting strategies in non-stationary linear bandits setting made by previous works and we propose a fix to their regret analysis.
\end{abstract}

\section{Introduction}

Reinforcement learning~\citep{sutton1998introduction} (RL) is
a framework for solving sequential decision-making problems. Through trial and error an agent must learn to act
optimally in an unknown environment in order to maximize
its expected reward signal. Efficient learning requires balancing exploration (acting to gather more information about the environment) and exploitation
(acting optimally according to the available knowledge).

One of the most popular principles that offer provably efficient exploration algorithms is \textit{Optimism in the face of uncertainty} (OFU). In tabular MDPs, the OFU principle has been successfully implemented~\citep{jaksch2010near, azar2017minimax}. Unfortunately, the performance of efficient tabular algorithms degrades with the number of states, which precludes applying them to arbitrarily large or continuous state spaces. An appealing challenge is to combine exploration strategies with generalization methods in a way that leads to both provable sample and computational efficient RL algorithms for large-scale problems. A straightforward way to ensure generalization over states is to aggregate them into a finite set of \textit{meta-states} and run tabular exploration mechanism on the latter. In this direction,~\citet{sinclair2019adaptive} and~\citet{touati2020zooming} propose to actively explore the state-action space by learning on-the-fly an adaptive partitioning that takes into account the shape of the optimal value function. When the state-action space is assumed to be a compact metric space, such adaptive discretization based algorithms yield sublinear regret but suffer from the curse of dimensionality as their regret scales almost exponentially with the covering dimension of the whole space. 

Another structural assumption, that received attention in the recent  literature~\citep{yang2019reinforcement, jin2020provably, zanette2020frequentist}, is when both reward and transition dynamics are linear functions with respect to a given feature mapping. This assumption enables the design of efficient algorithms with a linear representation of the action-value function. For example,~\citet{jin2020provably} propose \textsc{lsvi-ucb}, an optimisic modification of the popular least squares value iteration algorithm and achieve a $\widetilde{\mathcal{O}}(d^{3/2} H^2 K^{1 / 2})$ regret where $d$ is the dimension of the feature space, $H$ is the length of each episode and $K$ is the total number of episodes. However, most prior algorithms with linear function approximation assume that the environment is stationary and minimize the regret over the best fixed policy. While in many problems of interest, we are faced with a changing world, in some cases with substantial non-stationarity. This is a more challenging setting, since what has been learned in the past may be obsolete in the present.

In the present work, we study the problem of online learning in episodic non-stationary linear Markov Decision Processes (MDP), where both the reward and transition kernel are linear with respect to a given feature map and are allowed to evolve dynamically and even adversarially over time. The interaction of the agent with the environment is divided into $K$ episodes of fixed length $H$. Moreover, we assume that the total change of the MDP, that we measure by a suitable metric, over the $K$ episodes is upper bounded by $\Delta$, called \textit{variation budget}. 

To address this problem, we propose a computationally efficient model-free algorithm, that we call \textsc{opt-wlsvi}. We prove that, in the setting described above, its regret when competing against the best policy for each episode is at most $\widetilde{\mathcal{O}}(d^{5/4}H^2 \Delta^{1/4} K^{3/4})$. Concurrently to our work,~\citet{zhou2020nonstationary} propose to periodically restart \textsc{lsvi-ucb} from scratch, achieving the same regret. By contrast, our algorithm is based on weighted least squares value iteration that uses exponential weights to smoothly forget data that are far in the past, which drives the agent to keep exploring to discover changes. Our approach is motivated by the recent work of~\citet{russac2019weighted} who establish a new deviation inequality to sequential weighted least squares estimator and apply it to the non-stationary stochastic linear bandit problem. However, in contrast to linear bandit, our algorithm handles the additional problem of credit assignment since future states depend in non-trivial way on the agent's policy and thus we need  to carefully control how errors are propagated through iterations. Moreover, we discovered technical errors in the regret analysis of forgetting strategies in non-stationary linear bandits made by previous works and we propose a correction.

\section{Problem Statement}

\subsection{Notation}
Throughout the paper, all vectors are column vectors. We denote by $\norm{\cdot}$ the
Euclidean norm for vectors and the operator norm for matrices. For positive definite matrix $A$, we use $\norm{x}_A$ to denote the matrix norm $\sqrt{x^\top A x}$. We define $[N]$ to be the set
$\{1, 2, \ldots, N\}$ for any positive integer $N$.
\subsection{Non-Stationary Reinforcement Learning and Dynamic Regret}

We consider a non-stationary undiscounted finite-horizon MDP $ (\calS, \calA, P, r, H) $ where $\calS$ and $\calA$ are the state and action space, $H$ is the planning horizon i.e number of steps in each episode, $P = \{ P_{t, h}\}_{t > 0, h \in [H]}$  and $r = \{ r_{t, h}\}_{t > 0 , h \in [H]}$ are collections of transition kernels and reward functions, respectively. More precisely, when taking action $a$ in state $s$ at step $h$ of the $t$-th episode, the agent receives a reward $r_{t,h}(s, a)$ and makes a transition to the next state according to the probability measure $P_{t,h }(\cdot \mid s, a)$.

For any step $h \in [H]$ of an episode $t$ and $(s, a) \in \calS \times \calA$, the state-action value function of a policy $\pi = (\pi_1, \ldots, \pi_H)$ is defined as $
Q^{\pi}_{t, h}(s, a) = r_{t, h}(s, a) + \E \left[ \sum_{i=h+1}^H r_{t, i}(s_i, \pi_i(s_i)) ~\Big|~ s_h = s, a_h = a\right],
$ and the value function is $V_{t,h}^{\pi}(s) = Q^\pi_{t, h}(s, \pi_h(s))$.
The optimal value and action-value functions are defined as $V^\star_{t, h}(x) \triangleq \max_{\pi}V^{\pi}_{t, h}(s)$ and $Q^\star_{t, h}(s, a) \triangleq \max_{\pi}Q^{\pi}_{t, h}(s, a)$. If we denote $[P_{t,h} V_{t, h+1}](s, a) = \E_{s' \sim P_{t, h}(\cdot \mid s, a)}[V_{t, h+1}(s')]$, both $Q^\pi$ and $Q^\star$ can be conveniently written as the result of the following Bellman equations
\begin{align}
    Q^\pi_{t, h}(s, a) & = r_{t, h}(s, a) + [P_{t, h} V^\pi_{t, h+1}](s, a) \label{eq: bellman_eq}, \\
    Q^\star_{t, h}(s, a) & = r_{t, h}(s, a) + [P_{t, h} V^\star_{t, h+1}](s, a) \label{eq: bellman_opt},
\end{align}
where $V^\pi_{t, H+1}(s) = V^\star_{t, H+1}(s) = 0$ and $V^\star_{t, h}(s) = \max_{a \in \calA}Q^\star_{t, h}(s, a)$, for all $s \in \calS$.

\paragraph{Learning problem:} We focus on the online episodic reinforcement learning setting in which the rewards and the transition kernels are unknown. The learning agent plays the game for $K$ episodes $t=1, \ldots, K$, where each episode $t$ starts from some initial state $s_{t, 1}$ sampled according to some initial distribution. The agent controls the system by choosing a policy $\pi_t$ at the beginning of the $t$-th episode. We measure the agent's performance by the dynamic regret, defined as the sum over all episodes of the difference
between the optimal value function in episode $t$ and the value of $\pi_t$:
\begin{align*}
    \textsc{Regret}(K) = \sum_{t=1}^K V^*_{t,1}(s_{t, 1}) - V_{t,1}^{\pi_t}(s_{t, 1}).
\end{align*}
\subsection{Linear Markov Decision Processes}
In this work, we consider a special class of MDPs called linear MDPs, where both reward function and transition kernel can be represented as a linear function of a given feature mapping $\bphi: \calS \times \calA \rightarrow \R^d$. 
Now we present our main assumption

\begin{assumption}[Non-stationary linear MDP]
\label{assumption:non_stat_linear}

$(\calS, \calA, P, r, H) $ is non-stationary linear MDP with a feature map $\bphi: \calS \times \calA \rightarrow \R^d$ if for any $(t, h) \in \N \times [H]$, there exist $d$ unknown (signed) measures $\bmu_{t, h} = (\mu_{t, h}^{(1)}, \ldots, \mu_{t, h}^{(d)})$ over $\calS$ and an unknown vector $\btheta_{t, h} \in \R^d$, such that 
 for any $(s, a) \in \calS \times \calA$, we have 
\vspace{-2ex}
\begin{align} \label{eq:linear_transition} 
P_{t, h}(\cdot \mid s, a) &=  \bphi(s, a)^\top \bmu_{t, h}(\cdot) , \\
r_{t, h}(s, a) & = \bphi(s, a)^\top \btheta_{t, h}.  
\end{align}
Without loss of generality, we assume\footnote{A concrete case that would satisfy these assumptions, is if $\forall i \in [d], \bphi_i(s,a) \geq 0$ and $\sum_{i=1}^d \bphi_i(s,a)=1$, and $\forall i \in [d],\bmu_{t,h}^{(i)}$ is a probability measure. In this case $\bphi$(s,a) can be understood as providing the mixture coefficients with which to mix the $d$ measures in $\bmu_{t,h}$. } $\norm{\bphi(s, a)} \le 1$ for all $(s,a) \in \calS \times \calA$, and $\max\{\norm{\bmu_{t, h}(\calS)}, \norm{\btheta_{t, h}}\} \le \sqrt{d}$ for all $(t, h) \in \N \times [H]$.
\end{assumption}
Linear MDPs are also know as low-rank MDPs~\citep{zanette2020frequentist}. In fact, in the case of finite state and action spaces with cardinalities $|\calS|$ and $|\calA|$ respectively, the transition matrix $P \in \R^{(|\calS| \times |\calA|) \times |\calS|}$ could be expressed by the following low-rank factorization for any $(t, h)$: 
\begin{align*}
    P_{t,h} = \bm{\Phi} \bm{M}_{t,h}
\end{align*}
where $\bm{\Phi} \in \R^{(|\calS| \times |\calA|) \times d}$ such as $\bm{\Phi}[(s, a),:] = \bphi(s, a)^\top$ and $\bm{M}_{t,h} \in \R^{d \times |\calS|}$ such that $\bm{M}_{t,h}[:, s] = \bmu_{t,h}(s)$ a discrete measure. Therefore the rank of the matrix $P$ is at most $d$.

An important consequence of Assumption~\ref{assumption:non_stat_linear} is that the $Q$-function of any policy is linear in the features $\bphi$.

\begin{lemma} \label{lem: linear Q} For every policy $\pi$ and any $(t, h) \in \N^\star \times [H]$ there exists $\w^\pi_{t,h} \in \R^d$ such that 
\begin{equation}
    Q^\pi_{t,h} = \bphi(s, a)^\top \w^\pi_{t,h}, \quad \forall (s, a) \in \calS \times \calA.
\end{equation}
\end{lemma}

\section{The Proposed Algorithm}

Algorithm~\ref{algo:OPT-WLSVI}, referred as \textsc{opt-wlsvi} (OPTimistic Weighted Least Squares Value Iteration), parametrizes the Q-values $Q_{t, h}(s, a)$ by a linear form $\bphi(s, a)^\top \w_{t, h} $ and updates the parameters $\w_{t, h}$ by solving the following regularized weighted least squares problem:
\begin{align*}
    \w_{t, h} = \arg \min_{ \w \in \R^d} \Big \{  \sum_{\tau=1}^{t-1} \eta^{-\tau} \left( r_{\tau, h} + V_{t, h+1}(s_{\tau, h+1}) - \bphi_{\tau, h}^\top \w \right)^2 + \lambda \eta^{-(t-1)} \norm{\w}^2 \Big \}
\end{align*}
where $\eta \in (0, 1)$ is a discount factor, $V_{t, h+1}(s_{t, h+1}) = \max_{a \in \calA}Q_{t, h+1}(s_{\tau, h+1}, a)$ and $r_{\tau,h}$ and $\bphi_{\tau, h}$ are shorthand for $r_{\tau,h}(s_{\tau, h} , a_{\tau, h})$ and $\bphi(s_{\tau, h} , a_{\tau, h})$ respectively. The discount factor $\eta$ plays an important role as it gives exponentially increasing weights to recent transitions, hence, the past is smoothly forgotten.

The regularized weighted least-squares estimator of the above problem can be written in closed form
\begin{equation} \label{eq: wls_estimator}
    \w_{t,h} = \Sigma_{t, h}^{-1} \left( \sum_{\tau=1}^{t-1} \eta^{-\tau} \bphi_{\tau, h}(r_{\tau, h} + V_{t, h+1}(s_{\tau, h+1})) \right)
\end{equation}
where $\Sigma_{t, h} = \sum_{\tau =1}^{t-1}  \eta^{-\tau} \bphi_{\tau,h} \bphi_{\tau, h}^\top +  \lambda \eta^{- (t-1)} \cdot  \I$ is the Gram matrix. We further define the matrix
\begin{equation}
    \tlSigma_{t, h} = \sum_{\tau =1}^{t-1}  \eta^{-2\tau} \bphi_{\tau,h}\bphi_{\tau, h}^\top +  \lambda \eta^{-2(t-1)} \cdot \I
\end{equation}
The matrix $\tlSigma_{t, h}$ is connected to the variance of the estimator $w_{t, h}$, which involves the squares of the weights $\{ \eta^{-2\tau} \}_{\tau \geq 0}$. \textsc{opt-wlsvi} uses both matrices $\Sigma_{t, h}$ and $\tlSigma_{t, h}$ to define a upper confidence bound (UCB) term $\beta (\bphi^\top \Sigma_{t, h}^{-1} \tlSigma_{t, h}\Sigma_{t, h}^{-1}\bphi)^{1/2}$ to encourage exploration, where $\beta$ is a scalar.

The algorithm proceeds as follows. At the beginning of episode $t$, \textsc{opt-wlsvi} estimates the weighted least square estimator $w_{t,h}$ for each step $h \in [H]$ as given by Equation~\eqref{eq: wls_estimator}. Then, the algorithm updates the $Q$-value and the value function estimates as follows:
\begin{align*}
    Q_{t,h}(\cdot, \cdot) & = \bphi(\cdot, \cdot)^\top \w_{t,h} + \beta  (\bphi(\cdot, \cdot)^\top \Sigma_{t, h}^{-1} \tlSigma_{t, h} \Sigma_{t, h}^{-1} \bphi(\cdot, \cdot))^{1/2} \\
    V_{t, h}(\cdot) & = \min\{ \max_{a \in \calA} Q_{t,h}(\cdot, s), H \}  
\end{align*}
The UCB term is used to bound the estimation error of the value function, due to an insufficient number of samples, with high probability. The clipping of the value estimate is here to keep $V_{t, h}$ within the range of plausible values while preserving the optimism as $H$ is an upper bound on the true optimal value function. Finally the  algorithm collects a new trajectory by following the greedy policy $\pi_t$ with respect to the estimated $Q$-values.
\paragraph{Computational complexity:} At each step $h \in [H]$ of an episode $t \in [K]$, we need to compute the inverse of $\Sigma_{t,h}$ to solve the weighted least-squares problem. A naive implementation requires $\calO(d^3)$ elementary operations, but as $\Sigma_{t,h}$ is essentially a sum of rank-one matrices, we
need only $\calO(d^2)$ using the Sherman-Morrison update formula. Furthermore, $\calO(d^2)$ operations are needed to compute the exploration bonus $(\bphi^\top \Sigma_{t, h}^{-1} \tlSigma_{t, h}\Sigma_{t, h}^{-1}\bphi)^{1/2}$ that can be computed using only matrix-vector multiplications. Therefore computing $V_{t,h+1}$ for all the past successor states requires $\calO(d^2 |\calA| K)$ (the $|\calA|$ factor is due to the maximization over actions). As we need to do this at all steps and for every episode, the overall computation complexity of our algorithm is $\calO(d^2 |\calA| H K^2)$.

\begin{algorithm}[t]
\caption{Optimistic Weighted Least-Squares Value Iteration (\textsc{opt-wlsvi})}\label{algo:OPT-WLSVI}
\begin{algorithmic}[1]
\FOR{episode $t = 1, \ldots, K$}
\STATE Receive the initial state $s_{t, 1}$.
\STATE {\bf  \textcolor{gray!50!blue}{/* \texttt{Run LSVI procedure}}}
\STATE $V_{t, H+1}(\cdot) \leftarrow 0$
\FOR{step $h = H, \ldots, 1$}
\STATE $\w_{t,h} \leftarrow \Sigma_{t, h}^{-1} 
     ( \sum_{\tau=1}^{k-1} \eta^{-\tau} \bphi_{\tau, h}(r_{\tau, h} + V_{t, h+1}(s_{\tau, h+1})))$
\STATE $Q_{t,h}(\cdot, \cdot) \leftarrow \bphi(\cdot, \cdot)^\top \w_{t,h} + \beta  (\bphi(\cdot, \cdot)^\top \Sigma_{t, h}^{-1} \tlSigma_{t, h} \Sigma_{t, h}^{-1} \bphi(\cdot, \cdot))^{1/2}$ \label{line:ucb}
\STATE $V_{t, h}(\cdot) \leftarrow \min\{ \max_{a \in \calA} Q_{t,h}(\cdot, s), H \}$
\ENDFOR
\STATE \textbf{end for}
\STATE {\bf  \textcolor{gray!50!blue}{/* \texttt{Execute greedy policy}}}
\FOR{step $h = 1, \ldots, H$}
\STATE Execute $a_{t, h} = \argmax_{a \in \calA } Q_{t, h}(s_{t, h}, a)$
\STATE receive $r_{t, h}$ and observe $s_{t, h+1}$
\STATE {\bf  \textcolor{gray!50!blue}{/* \texttt{Update matrices}}}
\STATE $\Sigma_{t+1, h} \leftarrow \Sigma_{t, h} +   \eta^{-t} \bphi_{t,h}\bphi_{t, h}^\top +  \lambda \eta^{-t}(1- \eta) \cdot  \I$ \label{line:Sigma}
\STATE $\tlSigma_{t+1, h} \leftarrow \tlSigma_{t, h}+  \eta^{-2t} \bphi_{t,h}\bphi_{t, h}^\top +  \lambda \eta^{-2t} (1 - \eta^2 )\cdot \I$ \label{line:tlSigma}
\ENDFOR
\ENDFOR
\STATE \textbf{end for}
\end{algorithmic}
\end{algorithm}

\section{Non-stationary Linear Bandits}
Before providing the analysis of \textsc{opt-wlsvi}, we start by examining the linear bandit case when the horizon $H=1$. Let us first recall the non-stationary linear bandit model
\begin{definition}[Non-stationary linear bandit] Let $\calX \subset \R^d $ a set of decisions. At iteration $t$, the player makes a decision $x_t$ from a subset set $\calX_t \subset \calX$, then observes the reward $r_t$ satisfying:
\begin{equation}
    r_t = x_t^\top \btheta_t + z_t
\end{equation}
where $\btheta_t$ is the unknown regression parameter at iteration $t$ and $z_t$ is conditionally $\sigma$-subgaussian noise. We assume further that $\| x\| \leq 1, \forall x \in \calX$ and $\| \btheta_t\| \leq S, \forall t$.
\end{definition}
When $H=1$ linear MDP reduces to linear bandit if we let $\calX = \{ \bphi(s, a), a \in \calA, s \in \calS\}$ and $\calX_t = \{ \bphi(s_t, a), a \in \calA\}$ where $s_t$ is sampled from a given fixed distribution over states.

For the bandit setting, forgetting strategies have been proposed such as sliding-window, weighted regression and restarting in~\citep{cheung2019learning},~\citep{russac2019weighted} and~\citet{zhao2020simple} respectively. Randomized exploration with weighting strategy has also been introduced in~\citet{kim2020randomized}. The aforementioned works provide a regret of $\tilde{\calO}(d^{2/3} \Delta^{1/2} K^{2/3})$ which is optimal as it matches the established lower bound $\Omega (d^{2 /3} \Delta^{1 /3} K^{2 / 3})$ in~\citep{cheung2019learning} up to $\log(K)$ factors. Unfortunately, we find technical gaps in the regret analysis provided by the earliest paper~\citep{cheung2019learning}, which were then reproduced by the other three papers.
Specifically~\citet{cheung2019learning} attempted, in their Lemma 1, to upper bound the non-stationarity bias of the reward parameters by controlling the eigenvalues of matrix $M = V^{-1}_{t} \sum_{\tau=t-W}^{p} x_\tau x_\tau^\top$, where $V_{t} = \sum_{\tau=1}^{t-1} x_\tau x_\tau^\top + \lambda \cdot \I$ is the Gram matrix and for any integer $p \in \{t-W, \ldots, t-1\}$. They then needed to prove that $M$ is positive semi-definite, but their argument has technical errors. We precise in the appendix the issue in their argument and we provide concrete counter-examples.

Now, we provide a fix to the original error in the regret analysis of \textsc{sw-ucb} algorithm in~\citet{cheung2019learning} (see also Appendix~\ref{sec: d-linucb} for the analysis of~\textsc{d-linucb} algorithm proposed by~\citet{russac2019weighted}). At time $t$, \textsc{sw-ucb} selects a decision as follows: 
\begin{equation}
    x_t = \arg \max_{x \in \calX_t} x^\top \hat{\btheta}_t + \beta \norm{x}_{V_t^{-1}}
\end{equation}
where $\hat{\btheta}_t = V_t^{-1} \sum_{\tau=\max\{1, t-W \}}^{t-1} x_\tau r_\tau$ is the solution of the sliding window least squares problem

In their Lemma 1,~\citet{cheung2019learning} attempts to control the non-stationarity bias $\|\btheta_t - \bar{\btheta}_t \|$ where $\bar{\btheta}_{t} \triangleq V^{-1}_{t} \sum_{\tau=\max\{1, t-W \}}^{t-1} A_\tau A_\tau^\top \btheta_{\tau} + \lambda \btheta_t$ is the average of the true regression parameters over the sliding window. We propose to control $| x^{\top} (\btheta_t - \bar{\btheta}_t) |$ for any $x \in \calX$ and then use the fact that $ \norm{\btheta_t - \bar{\btheta}_t} = \max_{x: \norm{x} = 1} | x^{\top} (\btheta_t - \bar{\btheta}_t) |$

\begin{align*}
    | x^{\top} (\btheta_t - \bar{\btheta}_t) | 
    & = \left | x^{\top} V^{-1}_{t} \sum_{\tau=\max\{1, t-W \}}^{t-1} x_\tau x_\tau^\top (\btheta_\tau - \btheta_{t}) \right | \\
     & \leq \sum_{\tau=\max\{1, t-W \}}^{t-1} |   x^{\top} V_t^{-1} x_\tau | \cdot | x_\tau^\top (\sum_{s = \tau}^{t-1} (\btheta_{s} - \btheta_{s+1})) | \tag{triangle inequality } \\
     & \leq \sum_{\tau=\max\{1, t-W \}}^{t-1}  |  x^{\top} V_t^{-1} x_\tau | \cdot \| x_\tau\| \cdot \| \sum_{s = \tau}^{t-1} (\btheta_{s} - \btheta_{s+1})\| \tag{Cauchy-Schwarz}
     \\
     & \leq \sum_{\tau=\max\{1, t-W \}}^{t-1}   | x^{\top} V_t^{-1} x_\tau | \cdot   \sum_{s = \tau}^{t-1} \| \btheta_{s} - \btheta_{s+1}\| \tag{$\| x_\tau\| \leq 1$} \\
     & \leq \sum_{s =\max\{1, t-W \}}^{t-1} \sum_{\tau = \max\{1, t-W \}}^{s} | x^{\top} V_t^{-1} x_\tau | \cdot \| \btheta_{s} - \btheta_{s+1}\|
     \tag{$\sum_{\tau =\max\{1, t-W \}}^{t-1} \sum_{s=\tau}^{t-1} =  \sum_{s =\max\{1, t-W \}}^{t-1} \sum_{\tau =\max\{1, t-W \}}^{s} $} \\
     & \leq \sum_{s = \max\{1, t-W \}}^{t-1} \sqrt{ \bigg[ \sum_{\tau =\max\{1, t-W \}}^{s} x^\top V_{t}^{-1} x \bigg]  \cdot \biggl [ \sum_{\tau =\max\{1, t-W \}}^{s} x_{\tau}^\top V_{t}^{-1} x_{\tau}\bigg] }
     \cdot \norm{ \btheta_{s} - \btheta_{s+1}}
     \tag{Cauchy-Schwarz}
     \\ & \leq \sum_{s =\max\{1, t-W \}}^{t-1} \sqrt{ \bigg[ \sum_{\tau =\max\{1, t-W \}}^{s} x^\top V_{t}^{-1} x \bigg] \cdot d }
     \cdot \norm{ \btheta_{s} - \btheta_{s+1}}
     \tag{$(\star)$} \\
     & \leq \norm{x} \sqrt{d} \sum_{s =\max\{1, t-W \}}^{t-1} \sqrt{ \frac{\sum_{\tau =\max\{1, t-W \}}^{t-1} 1 }{ \lambda }} \cdot \norm{ \btheta_{s} - \btheta_{s+1}}  \tag{$\lambda_{\max}(V_t^{-1}) \leq \frac{1}{\lambda}$} \\
     & \leq \norm{x} \sqrt{\frac{d W}{\lambda }} \sum_{s =\max\{1, t-W \}}^{t-1} \norm{ \btheta_{s} - \btheta_{s+1}} 
\end{align*}

where the inequality $(\star)$ follows from the fact that $\sum_{\tau =\max\{1, t-W \}}^{s} x_{\tau}^\top V_{t}^{-1} x_{\tau} \leq d$ that can be proved as follows. We have $\sum_{\tau=\max\{1, t-W \}}^{t-1} x_{\tau}^{\top} V_t^{-1} x_{\tau} = \sum_{\tau=\max\{1, t-W \}}^{t-1} \text{tr}\left( x_\tau^\top V_t^{-1} x_\tau \right) = \text{tr}\left( V_t^{-1} \sum_{\tau=\max\{1, t-W \}}^{t-1} x_\tau x_\tau^\top \right)$. Given the eigenvalue decomposition $\sum_{\tau=\max\{1, t-W \}}^{t-1}  x_\tau x_\tau = \text{diag}(\lambda_1, \ldots, \lambda_d)^\top$, we have $V_t = \text{diag}(\lambda_1 + \lambda, \ldots, \lambda_d + \lambda)^\top$, and $\text{tr} \left( V_t^{-1} \sum_{\tau=1}^{t-1}  x_\tau x_\tau^\top\right) = \sum_{i=1}^d \frac{\lambda_j}{\lambda_j + \lambda } \leq d$.

Comparing to the bound on $\norm{\bar{\btheta}_t - \btheta_t}$ in the Lemma 1 of~\citet{cheung2019learning}, there is an extra factor $\sqrt{\frac{d W}{\lambda}}$ that multiplies the local non-stationarity term $\sum_{s = t - W}^{t-1} \norm{ \btheta_{s} - \btheta_{s+1}}$. This extra factor will consequently multiply the variation budget term in the final regret, as stated in the following proposition:

\begin{proposition}~\label{prop: sw-ucb regret} Under the assumption that $\sum_{t=1}^{K-1} \norm{\btheta_{t} - \btheta_{t+1}} \leq \Delta$, for any $\delta \in (0, 1)$, if we set $\beta = \sqrt{\lambda} S + \sigma \sqrt{2 \log(K/\delta) + d \log(1 + \frac{W}{\lambda d})}$ in the algorithm 1 \textsc{sw-ucb} of~\citet{cheung2019learning}, then with probability $1-\delta$, the dynamic regret of \textsc{sw-ucb} is at most 
\begin{align*}
\mathcal{O}\left(  \sqrt{\frac{d W}{\lambda}} \Delta W + \beta \sqrt{d K} \sqrt{\lceil K / W\rceil} \sqrt{ \log(1 + \frac{W}{d\lambda})} \right)
\end{align*}
\end{proposition}

Comparing to the regret upper bound in Theorem 3 of~\citet{cheung2019learning}, our fix leads to an extra factor $\sqrt{d W}$ multiplying the variation budget in , which becomes now $\tilde{\calO}(d^{1/2} \Delta W^{3/2} + d K W^{-1/2})$. Optimizing over the sliding window size $W$ leads to a final dynamic regret of $\tilde{\calO}( d^{7/8} \Delta^{1/4} K^{3/4})$. Note that the latter is not optimal since it does not match the lower bound $\Omega (d^{2 /3} \Delta^{1 /3} K^{2 / 3})$. This leaves the question of whether or not forgetting strategies are optimal to handle non-stationarity in linear bandits as an open research problem.

\section{Theoretical guarantee of \textsc{OPT-WLSVI}}
In this section, we present our main theoretical result which is an upper bound on the dynamic regret of \textsc{opt-wlsvi} (see Algorithm~\ref{algo:OPT-WLSVI}). First,
we quantify the variations on reward function and transition kernel over time in terms of their
respective variation budgets $\Delta_r$ and $\Delta_P$. The main advantage of using the variation budget is that it accounts for both slowly-varying and abruptly-changing MDPs.
\begin{definition}[MDP Variation budget]\label{def: variation} We define $\Delta = \Delta_r + \Delta_P$ where
\begin{align*}
    \Delta_r \triangleq  \sum_{t=1}^K \sum_{h=1}^H \norm{\btheta_{t, h} - \btheta_{t+1, h} }, \qquad
    \Delta_P  \triangleq \sum_{t=1}^K \sum_{h=1}^H \norm{ \bmu_{t,h}(\calS) - \bmu_{t+1, h}(\calS)}.
\end{align*}

\end{definition}

A similar notion has already been proposed in the literature, for instance total variance distance between $P_{t,h}$ and $P_{t+1,h}$ in tabular MDPs~\citep{ortner2019variational, cheung2020reinforcement}
or Wasserstein distance in smooth MDPs~\citep{domingues2020kernel}.

Now we present our bound on the dynamic regret for \textsc{opt-wlsvi}.
\begin{theorem}[Regret Bound] \label{theo: regret_bound} Under Assumption~\ref{assumption:non_stat_linear}, there exists an absolute constant $c > 0$ such that, for any fixed $\delta \in (0, 1)$, if we set $\lambda = 1$ and $\beta = c \cdot d H \sqrt{\imath}$ in Algorithm~\ref{algo:OPT-WLSVI} with $\imath \triangleq \log \left( \frac{2dH}{\delta (1-\eta)}\right)$, then with probability $1-\delta$, for any $W >0$ the dynamic regret of \textsc{opt-wlsvi} is at most
\begin{align}
    \mathcal{O} & \Big (
    c d^{3/2} H \sqrt{K \imath} \sqrt{2  K \log(1 / \eta) + 2  \log \left(1  +  \frac{1}{d \lambda (1-\eta)}\right)} +  H^{3/2} \sqrt{K \imath} + \underbrace{ \sqrt{\frac{d}{\lambda (1-\eta)}} H W \Delta + \frac{H^2 K \sqrt{d}}{\lambda}
     \frac{\eta^W}{1-\eta}}_{\text{non-stationarity bias}} \Big),
     \label{eq: regret_bound}
\end{align}
where $\mathcal{O}(\cdot)$ hides only absolute constants.
\end{theorem}

The last two terms of the of Equation~\eqref{eq: regret_bound} are the result of the bias due to the non-stationarity of the MDP. 
In theorem~\ref{theo: regret_bound} we introduce the parameter $W$ that can be interpreted, at a high level, as the effective temporal window equivalent to a particular choice of discount factor $\eta$: the bias resulting from transitions that are within the window $W$ may be bounded in term of $W$ while the remaining ones are bounded
globally by the last term of Equation~\eqref{eq: regret_bound}.

The following corollary shows that by optimizing the parameters $W$ and $\eta$, our algorithm achieves a sublinear regret. 
\begin{corollary} \label{coro: optimized_regret}
If we set $\log(1 / \eta) = \left( \frac{\Delta}{d K} \right)^{1/2}$ and $W = \frac{\log \left( K / (1-\eta) \right)}{\log(1/\eta)}$; under the same assumptions as in Theorem~\ref{theo: regret_bound}, for any $\delta \in (0, 1)$, we have that with probability $1-\delta$, the dynamic regret of \textsc{opt-wlsvi} is at most $\widetilde{\mathcal{O}}(d^{5/4} H^2 \Delta^{1/4} K^{3/4})$ where $\widetilde{\mathcal{O}}(\cdot)$ hides logarithmic factors.
\end{corollary}

In Corollary~\ref{coro: optimized_regret}, we rely on the knowledge of the variation budget $\Delta$ (or at least an upper bound on $\Delta$) in order to achieve a sublinear regret. We show in the next section how to relax the requirement of knowing the variation budget. In particular, we will describe how to extend our algorithm, using the Bandit-over-Reinforcement-Learning framework~\citep{cheung2020reinforcement} in order to deal with an unknown variation budget.


\begin{table*}[]
    \centering
    \begin{tabular}{c c c}
     Linear & Stationary & Non-stationary \\
     \hline
        \makecell{Bandits} & \makecell{$\widetilde{\mathcal{O}} (d K^{1 / 2})$ \\
        \citep{abbasi2011improved}} & \makecell{ $\widetilde{\mathcal{O}} (d^{7 /8} \Delta^{1 /4} K^{3 / 4})$ \\ \citet{cheung2019learning} \\ \citet{russac2019weighted} and our work} \\
        \hline 
         \makecell{MDPs} &
         \makecell{$\widetilde{\mathcal{O}} (d^{3 /2} H^2 K^{1 / 2})$
         \citep{jin2020provably} \\
         $\widetilde{\mathcal{O}} (d H^2 K^{1 / 2})$ \citep{zanette2020learning}}& \makecell{ $\widetilde{\mathcal{O}}(d^{5/4} H^2 \Delta^{1/4} K^{3/4})$ \\ Our work } \\
        \hline
    \end{tabular}
    \caption{ Comparison of our regret bound with state-of-the-art bounds for both linear bandits and linear MDPs. $d$ is the dimension of the features space, $H$ is the planning horizon of the MDP, $K$ is the number of episodes and $\Delta$ is the variation budget. When we go from a bandit setting to MDPs, the work of~\citet{jin2020provably} in the stationary case and our work in the non-stationary case incur an extra $d^{1/2}$ factor and $d^{3/8}$ respectively.~\citet{zanette2020learning} achieve a linear dependence on $d$ in the stationary case but their proposed algorithm is computationally intractable.}
    \label{tab: bounds}
\end{table*}

\subsection{Unknown variation budget}

Our algorithm \textsc{opt-wlsvi} needs the variation budget $\Delta$ to set the optimal value of the forgetting parameter as $\log(1/ \eta^\star) = \left( \frac{\Delta}{d K} \right)^{1/2}$. We can use the Bandit-over-Reinforcement-Learning framework (BoRL)~\citep{cheung2020reinforcement} to tune the forgetting parameter online. 

The idea is to run a multi-armed bandit algorithm over a set of sub-algorithm each using a different parameter. In our case, each sub-algorithm is a \textsc{opt-wlsvi} with a different guess on $\eta^\star$. If $\Delta \geq \sqrt{K}$ the regret bound is vacuous (linear regret), we are only interested in problems with  $\Delta$ in the range $[1, \sqrt{K}]$. This implies that the set of $\log(1/ \eta)$ only needs to span the range $[ \frac{1}{\sqrt{dK}}, \frac{1}{\sqrt{d}}]$.

We divide the horizon $K$ into $\frac{K}{M}$
equal-length intervals each of length $M$, specified later. In each interval, sub-algorithm $i$ restarts a \textsc{opt-wlsvi} with $\log(1/ \eta_i) = \frac{2^i}{\sqrt{d K}}$. We have in total $I = \lfloor \log_2(\sqrt{K}) \rfloor + 1$ possible value of $\log(1/ \eta)$ in the form of $\frac{2^i}{\sqrt{d K}}$ that spans $[ \frac{1}{\sqrt{dK}}, \frac{1}{\sqrt{d}}]$. We can verify that there exists $i^\star \in [A]$ such that $\log(1/ \eta_{i^\star}) \leq \log(1/ \eta^\star) \leq 2 \log(1/ \eta_{i^\star})$, which well-approximates the
optimal parameter up to constant factors.

On top of these sub-algorithms, we run the \textsc{exp3.p}~\citep{auer2002nonstochastic}. The arms are the sub-algorithms. There are $I$ arms and the reward for each arm or sub-algorithm $i$ in interval $m \in [\frac{K}{M}]$ is the total of reward collected in the MDP during this interval. \textsc{exp3.p} is called for $\frac{K}{M}$ rounds to select the arm.

\paragraph{Regret Analysis of \textsc{opt-wlsvi + BoRL}:}

Let $i_m$ the arm selected by \textsc{exp3.p} for the interval $m \in [\frac{K}{M}]$ and $\pi^i$ is the algorithm followed by a sub-algorithm $i$. The regret of the overall algorithm can be decomposed as the regret of the algorithm $i^\star$ that optimally tunes the parameter plus the loss due to learning $i^\star$ with the \textsc{exp3.p} algorithm: 

\begin{align*}
\textsc{Regret}(K) &= \left( \sum_{t=1}^K V^\star_{t,1}(s^1_{t}) - V_{t,1}^{\pi^{i^\star}_t}(s^1_t) \right) + \left(\sum_{m=1}^{\frac{K}{M}} \sum_{t = (m-1)M + 1}^{m M} V_{t,1}^{\pi_t^{i^\star}}(s^1_{t}) -
    V_{t,1}^{\pi_t^{i_m}}(s^1_{t}) \right)
\end{align*}

The fist term corresponds to \textsc{opt-wlsvi} with parameter $\eta_{i^\star}$. Therefore, we can bound this term using Theorem 6.2. As $\eta_{i^\star}$ differs from $\eta^\star$ up to constant factor, we obtain the bound in Corollary 6.3 i.e $\tilde{\mathcal{O}}(d^{5/4}H^2 \Delta^{1/4} K^{3/4})$.

The second term corresponds to the regret of the \textsc{exp3.p} learner against the sub-algorithm $i^\star$. There are $I$ arms, \textsc{exp3.p} is called for $\frac{K}{M}$ rounds and the rewards collected during each interval is upperbounded by $M H$. Therefore, by a classical regret bound of \textsc{exp3.p}~\citep{auer2002nonstochastic}, the second term is upper-bounded with high probability by:
\begin{equation*}
    \tilde{\mathcal{O}}(M H \sqrt{I \frac{K}{M}}) = \tilde{\mathcal{O}}( H \sqrt{M K})
\end{equation*}

We obtain that $\textsc{regret}(K) =  \tilde{\mathcal{O}}( d^{5/4}H^2 \Delta^{1/4} K^{3/4} + H \sqrt{M K} )$ and by choosing $M = \sqrt{K}$, we conclude that $\textsc{regret}(K) =  \tilde{\mathcal{O}}( d^{5/4}H^2 \Delta^{1/4} K^{3/4})$. 
Note that we obtain the same regret bound when the variation budget is known.

\section{Technical Highlights}
In this section,  we give an overview of some key ideas leading to the regret bound in Theorem~\ref{theo: regret_bound}. Inspired by the analysis of weighting approach in bandit~\citep{russac2019weighted}, one can attempt to interpret the algorithm as acting optimistically with respect to the weighted parameters of the optimal Q-value defined as 
$\bar{\w}_{t, h}(s, a)= \Sigma_{t, h}^{-1}
    ( \sum_{\tau = 1}^{t-1} \eta^{-\tau} \bphi_{\tau, h} \bphi_{\tau, h}^\top \w^\star_{\tau, h} + \lambda \eta^{-(t-1)} \w^\star_{t, h})$, 
    where 
$\w^\star_{t,h}$ 
are the true parameters of the optimal Q-value. This first attempt was unsuccessful. Then, we came up with the implicitly defined \textit{weighted MDP} and we were able to interpret our algorithm as acting optimistically with respect to this weighted MDP. We provide the full proofs and derivations in the appendix. We first translate the parameter update produced by the algorithm into the following compact update of $Q$-value estimates for any $t \in [K]$ and $h \in \{H, \ldots 1\}$:
\begin{equation}\label{eq: emp_backward}
    Q_{t,h} = \widehat{r}_{t,h} + \widehat{P}_{t, h}V_{t, h+1} + B_{t, h}
\end{equation}
where we define the implicitly empirical reward function $\widehat{r}$ and transition measure $\widehat{P}$ as follows: 
\begin{align*}
    \widehat{r}_{t, h}(s, a) &\triangleq \bphi(s, a)^\top \Sigma_{t, h}^{-1} ( \sum_{\tau=1}^{t-1} \eta^{-\tau} \bphi_{\tau, h} r_{\tau, h} ), \\
    \widehat{P}_{t, h}(\cdot \mid s, a) &\triangleq \bphi(s, a)^\top  \Sigma_{t, h}^{-1} ( \sum_{\tau=1}^{t-1} \eta^{-\tau} \bphi_{\tau, h}
    \delta(\cdot, s_{\tau, h+1})),
\end{align*}
and $B_{t, h}(\cdot , \cdot) = \beta ( \bphi(\cdot, \cdot)^\top \Sigma_{t, h}^{-1} \tlSigma_{t, h} \Sigma_{t, h}^{-1} \bphi(\cdot, \cdot))^{1/2} = \beta \norm{\bphi(\cdot , \cdot )}_{\Sigma_{t, h}^{-1} \tlSigma_{t, h} \Sigma_{t, h}^{-1}}$ is the exploration bonus.

 We can interpret the Equation~\eqref{eq: emp_backward} as an approximation of the backward induction in a \textit{weighted average} MDP defined formally as follows.
\begin{definition}[Weighted Average MDP] let for any $(s, a) \in \calS \times \calA$,
\begin{align*}
    \bar{r}_{t, h}(s, a) &\triangleq \bphi(s, a)^\top \Sigma_{t, h}^{-1}
    ( \sum_{\tau = 1}^{t-1} \eta^{-\tau} \bphi_{\tau, h} \bphi_{\tau, h}^\top \btheta_{\tau, h} + \lambda \eta^{-(t-1)} \btheta_{t, h}), \\
    \bar{P}_{t, h}(\cdot \mid s, a) & \triangleq \bphi(s, a)^\top \Sigma_{t, h}^{-1} 
    ( \sum_{\tau = 1}^{t-1} \eta^{-\tau}  \bphi_{\tau, h} \bphi_{\tau, h}^\top \bmu_{\tau, h}(\cdot ) + \lambda \eta^{-(t-1)} \bmu_{t, h}(\cdot)).
\end{align*}
$(\calS, \calA, \bar{P}, \bar{r})$ is called the weighted average MDP.
\end{definition}
We can see that if we ignore the regularization term (we set $\lambda$ to zero), $\widehat{r}_{t,h}$ coincides with $\bar{r}_{t,h}$ and $\widehat{P}_{t,h}$ is an unbiased estimate of $\bar{P}_{t,h}$. Therefore, in contrast with the stationary case, we are tracking the $Q$-value of the weighted average MDP instead of the true MDP at time $t$. The next Lemma quantifies the bias arising from the time variations of the environment.

\begin{lemma}[Non-stationarity bias] \label{lemma: bias} For any $ W \in [t-1]$ and for any bounded function $ f: \calS \rightarrow \R$ such as $\norm{f}_{\infty} \leq H$, we have:
\begin{align*}
    |r_{t, h}(s, a) - \bar{r}_{t, h}(s, a)|  \leq \texttt{bias}_{r}(t, h)  , \quad
     \Big | [(P_{t, h} -  \bar{P}_{t,h})f](s, a) \Big | \leq H \texttt{bias}_{P}(t, h),
\end{align*}
where
\begin{align*}
\texttt{bias}_{r}(t, h) & = \sqrt{\frac{d}{\lambda (1-\eta)}}
    \sum_{s = t - W}^{t-1}  \| \btheta_{s, h} - \btheta_{s+1, h}\| +
     \frac{2 \sqrt{d} \eta^W}{\lambda(1- \eta)}, \\
    \texttt{bias}_{P}(t, h) & = \sqrt{\frac{d}{\lambda (1-\eta)}} \sum_{s = t - W}^{t-1} \left \| \bmu_{s,h}(\calS) - \bmu_{s+1, h}(\calS) \right\| + \frac{2 \sqrt{d} \eta^W}{\lambda(1- \eta)}.
\end{align*}
\end{lemma}

We analyse now the one-step error decomposition of the difference between the estimates $Q_{t,h}$ and $Q^\pi_{t,h}$ of a given policy $\pi$. To do that, we use the weighted MDP $(\calS, \calA, \bar{P}, \bar{r})$ to isolate the bias term. The decomposition contains four parts: the reward bias and variance, the transition bias and variance, and the difference in value functions at step $h+1$. It can be written as:
\begin{align*}
    & \bphi(s, a)^\top \w_{t, h} - Q^\pi_{t,h}(s, a) =  \underbrace{(\bar{r}_{t,h} - r_{t,h})(s, a)}_{\text{reward bias}} +  \underbrace{(\widehat{r}_{t,h} - \bar{r}_{t,h})(s, a)}_{\text{reward variance}} + \\
& \quad \underbrace{[(\bar{P}_{t,h} - P_{t,h})V^\pi_{t, h+1}](s, a)}_{\text{transition bias}} +  \underbrace{[(\widehat{P}_{t,h} - \bar{P}_{t,h})V_{t,h}](s, a)}_{\text{transition variance}} + \underbrace{[\bar{P}_{t, h}(V_{t, h+1} - V^\pi_{t, h+1})](s, a)}_{\text{difference in value functions of next step}}.
\end{align*}
This differs from the error decomposition in the analysis of \textsc{LSVI-UCB} in several aspects: firstly, the variance terms are with respect the newly defined weighted MDP and not the true MPD. Secondly, we have additional reward and transition bias terms. Finally, the difference in the difference in value-functions at step $h+1$ hides also another bias term. Therefore, we need to carefully propagate bias terms through iteration. 

The reward and transition bias terms are controlled by Lemma~\ref{lemma: bias} using the fact that $\| V^\pi_{t,h}\|_\infty \leq H$. The difference in value-functions at step $h+1$ can be rewritten as $ [P_{t,h}(V_{t,h+1} - V^\pi_{t, h+1})](s, a) + [(\bar{P}_{t,h} - P_{t, h})(V_{t,h+1} - V^\pi_{t, h+1})](s, a)$. We control the second term by applying again Lemma~\ref{lemma: bias} since $\| V_{t,h+1} - V^\pi_{t, h+1}\|_\infty \leq H$.

It remains now the two variance terms. The reward variance is easy to control and it reduces simply to the bias due to the regularization as we assume that $r$ is a deterministic function. Note that the assumption of deterministic reward is not a limiting assumption since the contribution of a stochastic reward in the final regret has lower order term than the contribution of a stochastic transition. 
Controlling the transition variance is more involved. Basically, we would like use the concentration of weighted self-normalized processes~\citep{russac2019weighted} to get a high probability bound. However, as $V_{t, h+1}$ is estimated from past transitions and thus depends on the latter in a non-trivial way, we show a concentration bound that holds uniformly for all possible value functions generated by the algorithm. This done by using a union bound argument over an $\epsilon$-net of the set of possible value functions with an appropriate value of $\epsilon$. 

\begin{lemma}
For any $\delta \in (0, 1)$, with probability at least $1 - \delta/2$, we have for all $(t, h) \in [K] \times [H]$,
\begin{align*}
\norm{\sum_{\tau=1}^{t-1}  \eta^{-\tau}\bphi_{\tau, h} \epsilon_{\tau, h}}_{\tlSigma_{t,h}^{-1}} \leq C d H  \sqrt{ \log \left(\frac{d H \beta}{\lambda (1 - \eta)}  \cdot \frac{2}{\delta}\right)}
\end{align*}
where $C >0$ is an absolute constant. 
\end{lemma}
Let $\texttt{bias} \triangleq \texttt{bias}_{r} + \texttt{bias}_{P}$ the total non-stationarity bias of the MDP. By deriving an appropriate value of $\beta$ (see Lemma~\ref{lemma: key lemma}) and an induction arguments, we establish the optimism of our value estimates.
\begin{lemma}[Optimism]
There exists an absolute value $c$ such that $\beta = c d H \sqrt{\imath}$ where $\imath = \log \left( \frac{2 d H}{(1-\eta) \delta} \right)$, $\lambda =1$ and for all $(s, a, t, h) \in \calS \times \calA \times [K] \times [H]$, we have with probability at least $1 - \delta / 2$
\begin{align}
    Q_{t,h}(s, a) + 2 H \sum_{h' =h}^H \texttt{bias}(t, h) \geq Q_{t, h}^\star(s, a)
\end{align}
\end{lemma}

\section{Related Work}
\paragraph{RL with linear function approximation: } Provable algorithms with linear function approximation have seen a growing research interest in the recent literature. Under the assumption of stationary linear MDP,~\citet{jin2020provably} propose an optimistic version of LSVI (\textsc{lsvi-ucb}) that achieves a regret of $\widetilde{\calO} (d^{3 /2} H^2 K^{1 / 2})$ where the exploration is induced by adding a UCB bonus to the estimate of the action-value function. Whereas ~\citet{zanette2020frequentist} introduce a randomized version of LVSI that achieves  $\widetilde{\calO} (d^{2} H^2 K^{1 / 2})$ regret where the exploration is induced by perturbing the estimate of the action-value function. Lately,~\citet{zanette2020learning} consider a more general assumption, zero inherent Bellman error, which states that the space of linear functions is close with respect to the Bellman operator (Note that linear MDPs have zero inherent Bellman error). Instead of adding UCB bonuses for every experienced states at each step $h \in [H]$, they propose to solve a global planning optimization program that returns an optimistic solution at the initial state, achieving $\widetilde{\calO} (d H^2 K^{1 / 2})$ regret.~\citet{yang2019reinforcement} study a slightly different assumption where the transition kernel admits a three-factor low-rank factorization $P(\cdot \mid \cdot) = \bphi(\cdot)^\top M^\star \bpsi(\cdot)$. They propose a model-based algorithm that tries to learn the \textit{core matrix} $M^\star$ and they show that it achieves $\widetilde{\calO} (d H^2 K^{1 / 2})$ regret.

Linear function approximations have also been studied in adversarial settings, where the reward function is allowed to change between episodes in an adversarial manner but the transition kernel stays the same. In the full-information setting,~\citet{cai2019provably} propose an optimistic policy optimization algorithm that achieves $\widetilde{\calO} (d H^2 K^{1 / 2})$ providing that the transition kernel has a linear structure $P_h(s' \mid s, a) = \bpsi(s, a, s')^\top \theta_h$. In the bandit feedback setting,~\citet{neu2020online} propose a new algorithm based on adversarial linear bandit that achieves $\widetilde{\calO} ( (d |\calA|)^{1/3} H^2 K^{2 / 3})$ regret under the assumption that all action-value functions can be represented as linear functions.

Concurrently to our work,~\cite{zhou2020nonstationary} also study non-stationary linear MDPs. They establish a lower bound of $\Omega( d^{2/3} H^2  \Delta^{1/3} K^{2/3})$ and they propose a restart strategy that achieves the same dynamic regret as in our Corollary~\ref{coro: optimized_regret}. Their algorithm consists in restarting periodically \textsc{lsvi-ucb}, and is thus markedly different from our approach. By throwing away historical data from time to time such a restart strategy would be best suited for abrupt changes in the environment, whereas our approach, by smoothly forgetting the past, would be more beneficial for gradually changing environments. Empirical comparison of both strategies in the bandit setting~\citep{zhao2020simple} confirm this.

\textbf{Non-stationary RL:} Provably efficient algorithms for non-stationary RL in the tabular case have been introduced in several recent works. While~\citet{gajane2018sliding} and~\citet{cheung2020reinforcement} use a sliding-window approach, ~\citet{ortner2019variational} implement a restart strategy where at each restart, past observations are discarded and new estimators for the reward and the transition kernel are built from scratch. Very recently,~\citet{domingues2020kernel} tackle non-stationary RL in  continuous environments, where rewards and transition kernel are assumed to be Lipschitz with respect to some similarity metric over the state-action space. They propose a kernel-based algorithm with regret guarantee using time and space dependant smoothing kernels.

\section{Recent Developments}
The authors of~\citet{cheung2019learning} who pioneered the non-stationary linear bandit, released recently a revised version of their AISTATS 2019 paper to acknowledge the mistake in the analysis. In order for their optimal rate $\tilde{\mathcal{O}}(T^{2/3})$ to hold, they assume that actions are orthogonal.

~\cite{zhao2021non} identify also the mistake and proposed the same fix than ours in a technical note released slightly after we had made public on arxiv a version of our paper containing the fix.
While the fix strategies are found independently, we would like to credit to Peng Zhao for detecting this technical gap in the first place. 

While our work shows that forgetting strategies achieve the rate of $\tilde{\mathcal{O}}(T^{3/4})$, the recent work~\citet{Wei2021NonstationaryRL} follows a substantially different approach to achieve the optimal rate $\tilde{\mathcal{O}}(T^{2/3})$ for the first time in the setting of non-stationary linear bandit and MDP. Their algorithm detects non-stationarity by running multiple instances of a base (stationary) algorithm with different durations in a randomized schedule.

\section{Conclusion}
In this paper, we studied the problem of RL with linear function approximation in a changing environment where the reward and the transition kernel can change from time to time as long as the total changes are bounded by some variation budget. We introduced a provably efficient algorithm in this setting. The algorithm uses a discount factor to reduce the influence of the past and estimates the $Q$-value's parameters through weighted LSVI. We revisited as well the linear bandit setting. We pointed out a serious technical problem in the analysis of all forgetting strategies. Then, we provide a new regret analysis of these algorithms.

\textbf{Limitations:} In order to obtain theoretical guarantees, we need to make some assumptions such as Linear MDPs.  Assumptions weaker than linear MDP either result in computationally inefficient algorithms (as in~\citet{zanette2020learning}) or require the transition to be deterministic~\citep{du2020agnostic}.
Furthermore, our $\widetilde{\mathcal{O}}(T^{3/4})$ regrets for both bandits and MDPs don't match the $\Omega(T^{2/3})$ lower bounds for these problems. Forgetting strategies have been mistakenly believed optimal in linear bandits. In contrast, our work shows that the latter is not true and leaves the question of minimax rate open again. It is an interesting direction to explore in future work.
\bibliographystyle{apalike}
\bibliography{lib}
\newpage
\appendix

\section{Technical Gaps in Published Bandit Papers}
In this section,
we highlight the technical error made by~\citep{cheung2019learning} when controlling the bias term due to the non-stationarity of the reward function. Let us first recall the non-stationary linear bandit model
\begin{definition}[Non-stationary linear bandit] At iteration $t$, the player makes a decision $A_t$ from a feasible set $\calA \subset \R^d$, then observes the reward $r_t$ satisfying:
\begin{equation}
    r_t = A_t^\top \btheta_t + z_t
\end{equation}
where $\btheta_t$ is the unknown regression parameter at iteration $t$ and $z_t$ is conditionally $\sigma$-subgaussian noise. We assume further that $\| A\| \leq 1, \forall A \in \calA$ and $\| \btheta_t\| \leq S, \forall t$.
\end{definition}

\citet{cheung2019learning} propose the \textsc{sw-ucb} algorithm based on a sliding window approach of size $W$. At time $t$, actions are selected as follows: 
\begin{equation}
    A_t = \arg \max_{a \in \calA} a^\top \hat{\btheta}_t + \beta \norm{a}_{V_t^{-1}}
\end{equation}
where $\hat{\btheta}_t$ is the solution of the sliding window least squares problem:
\begin{equation}
    \hat{\btheta}_t = V_t^{-1} \sum_{\tau=\max\{1, t-W \}}^{t-1} A_\tau r_\tau \quad \text{, where} \quad V_t = \sum_{\tau=\max\{1, t-W \}}^{t-1} A_\tau A_\tau^\top + \lambda \cdot \I \text{ is the Gram matrix.}
\end{equation}

In the proof of lemma 1 in~\citet{cheung2019learning}, the authors consider matrix $M =  V^{-1}_t X$ where $X = \sum_{\tau=t-W}^{p} A_\tau A_\tau^\top$ for any integer $p \in \{t-W, \ldots, t-1\}$. They attempt to show that $M$ is positive semi-definite (PSD) (i.e $y^\top M y \geq 0, \forall y \in \R^d$) as follows: they first prove that $M$ shares the same  characteristic polynomial as the matrix $V^{-1/2}_t X V^{-1/2}$, then assert that since $V^{-1/2}_t X V^{-1/2}$ is PSD, $M$ is PSD as well. 

Unfortunately, this last assertion does not hold in general. As a counterexample, let us consider the 2 dimensional identity matrix $\I$ and $B = ((1,0)^\top, (-10, 1)^\top)$. $\I$ and $B$ share the same characteristic polynomial $p(x) = (x-1)^2$, $\I$ is obviously PSD but $B$ is not, as for $y = (1, 1)^\top$, we have $y^\top B y = -8 < 0$. 

Moreover, in general, a matrix of the form $M = V^{-1}_t X$ is not guaranteed to be PSD. If one sets $d=2$, $t=3$, $\lambda=1$, $A_1 = (1,0)^\top$ and $A_2 = (1, 1)^\top$. with $A_1 = (1,0)^\top$ and $A_2 = (1, 1)^\top$, we have $M = V^{-1}_t A_1A_1^\top = ((0.4, - 0.2)^\top, (0, 0)^\top)$. If we consider $y = (1,5)^\top$, we have $y^\top M y = -0.6 < 0$.

\section{Regret Reanalysis of \textsc{d-linucb}} 
\label{sec: d-linucb}
\citet{russac2019weighted} propose the \textsc{d-linucb} algorithm, based on sequential weighted least squares regression. At time $t$, actions are selected as follows: 
\begin{equation}
    x_t = \arg \max_{x \in \calX_t} x^\top \hat{\btheta}_t + \beta \norm{x}_{V_t^{-1} \tilde{V}_t V_t^{-1} }
\end{equation}
where $V_t = \sum_{\tau=1}^{t-1} \eta^{-\tau} x_\tau x_\tau^\top + \lambda \eta^{-(t-1)} \cdot \I$ is the Gram matrix, $\tilde{V}_t = \sum_{\tau=1}^{t-1} \eta^{-2 \tau} x_\tau x_\tau^\top + \lambda \eta^{-2 (t-1)} \cdot \I$ and $\hat{\btheta}_t$ is the solution the weighted least squares problem:
\begin{equation}
    \hat{\btheta}_t = V_t^{-1} \sum_{\tau=1}^{t-1} \eta^{-\tau} x_\tau r_\tau.
\end{equation}

As our analysis follows the same proof steps as in~\citet{russac2019weighted}, we will only highlight our proposed fix to their technical error and the changes that it induces.
\paragraph{Non-stationarity bias}
Let $\bar{\btheta}_{t} \triangleq V^{-1}_{t} \sum_{\tau=1}^{t-1} \eta^{-\tau} x_\tau x_\tau^\top \btheta_{\tau} + \lambda \eta^{-(t-1)} \btheta_t$ the weighted average of the true regression parameters. To characterize the     bias,~\citet{russac2019weighted} attempt to control directly $\|\btheta_t - \bar{\btheta}_t \|$. Instead, we propose to control $| x^{\top} (\btheta_t - \bar{\btheta}_t) |$ for any $x \in \calX$ and then use the fact that $ \norm{\btheta_t - \bar{\btheta}_t} = \max_{x: \norm{x} = 1} | x^{\top} (\btheta_t - \bar{\btheta}_t) |$
\begin{align*}
    | x^{\top} (\btheta_t - \bar{\btheta}_t) |& = \left |x^{\top} V_t^{-1} \sum_{\tau=1}^{t-1} \eta^{-\tau}  x_\tau x_\tau^\top (\btheta_{\tau} - \btheta_t) \right | \\
    & \leq \underbrace{\left | x^{\top} V_t^{-1} \sum_{\tau=t-W}^{t-1} \eta^{-\tau}  x_\tau x_\tau^\top (\btheta_{\tau} - \btheta_t) \right |}_{(\star)} + \underbrace{
    \left | x^{\top} V_t^{-1} \sum_{\tau=1}^{t-W-1} \eta^{-\tau}  x_\tau x_\tau^\top (\btheta_{\tau} - \btheta_t)\right | }_{(\star \star)} 
\end{align*}
\paragraph{Bound on $(\star)$:} 
\begin{align*}
    & \left | x^{\top} V_t^{-1} \sum_{\tau=t-W}^{t-1} \eta^{-\tau}  x_\tau x_\tau^\top (\btheta_{\tau} - \btheta_t) \right | \\
    & \leq   \sum_{\tau=t-W}^{t-1} \eta^{-\tau} \left |   x^{\top} V_t^{-1} x_\tau \right | \cdot | x_\tau^\top (\btheta_{\tau} - \btheta_t) | \tag{triangle inequality } \\
     & = \sum_{\tau=t-W}^{t-1} \eta^{-\tau}  |   x^{\top} V_t^{-1} x_\tau | \cdot | x_\tau^\top (\sum_{s = \tau}^{t-1} (\btheta_{s} - \btheta_{s+1})) | \\
     & \leq \sum_{\tau=t-W}^{t-1} \eta^{-\tau} |  x^{\top} V_t^{-1} x_\tau | \cdot \| x_\tau\| \cdot \| \sum_{s = \tau}^{t-1} (\btheta_{s} - \btheta_{s+1})\| \tag{Cauchy-Schwarz}
     \\
     & \leq \sum_{\tau=t-W}^{t-1}  \eta^{-\tau}  | x^{\top} V_t^{-1} x_\tau | \cdot   \sum_{s = \tau}^{t-1} \| \btheta_{s} - \btheta_{s+1}\| \tag{$\| x_\tau\| \leq 1$} \\
     & \leq \sum_{s = t - W}^{t-1} \sum_{\tau = t-W}^{s} \eta^{-\tau}  | x^{\top} V_t^{-1} x_\tau | \cdot \| \btheta_{s} - \btheta_{s+1}\|
     \tag{$\sum_{\tau = t - W}^{t-1} \sum_{s=\tau}^{t-1} =  \sum_{s = t - W}^{t-1} \sum_{\tau = t-W}^{s} $} \\
     & \leq \sum_{s = t - W}^{t-1} \sqrt{ \bigg[ \sum_{\tau = t-W}^{s} \eta^{-\tau} x^\top V_{t}^{-1} x \bigg]  \cdot \biggl [ \sum_{\tau = t-W}^{s}  \eta^{-\tau} x_{\tau}^\top V_{t}^{-1} x_{\tau}\bigg] }
     \cdot \norm{ \btheta_{s} - \btheta_{s+1}}
     \tag{Cauchy-Schwarz}
     \\ & \leq \sum_{s = t - W}^{t-1} \sqrt{ \bigg[ \sum_{\tau = t-W}^{s} \eta^{-\tau} x^\top V_{t}^{-1} x \bigg] \cdot d }
     \cdot \norm{ \btheta_{s} - \btheta_{s+1}}
     \tag{by lemma~\ref{lem:basic_ineq}} \\
     & \leq \norm{x} \sqrt{d} \sum_{s = t - W}^{t-1} \sqrt{ \frac{\sum_{\tau = t-W}^{t-1} \eta^{-\tau}}{ \lambda \eta^{-(t-1)}}} \cdot \norm{ \btheta_{s} - \btheta_{s+1}}  \tag{$\lambda_{\max}(V_t^{-1}) \leq \frac{1}{\lambda \eta^{-(t-1)}}$} \\
     & \leq \norm{x} \sqrt{\frac{d}{\lambda (1-\eta)}} \sum_{s = t - W}^{t-1} \norm{ \btheta_{s} - \btheta_{s+1}} 
\end{align*}

\paragraph{Bound on $(\star \star)$:}
\begin{align*}
    \left | x^{\top} V_t^{-1} \sum_{\tau=1}^{t-W-1} \eta^{-\tau}  x_\tau x_\tau^\top (\btheta_{\tau} - \btheta_t)\right | & \leq \norm{x} \norm{V_t^{-1} \sum_{\tau=1}^{t-W-1} \eta^{-\tau}  x_\tau x_\tau^\top (\btheta_{\tau} - \btheta_t)} \\
    & \leq \norm{x} \frac{1}{\lambda \eta^{-(t-1)}} \norm{ \sum_{\tau=1}^{t-W-1} \eta^{-\tau}  x_\tau x_\tau^\top (\btheta_{\tau} - \btheta_t)}
    \tag{$\norm{V_t^{-1}} = \lambda_{\max}(V_t^{-1}) \leq \frac{1}{\lambda \eta^{-(t-1)}}$} \\
    & \leq \norm{x} \frac{1}{\lambda} \sum_{\tau=1}^{t-W-1} \eta^{(t-1-\tau)} \norm{x_\tau}^2 \norm{\btheta_{\tau} - \btheta_t} \\
    & \leq \norm{x} \frac{2 S}{\lambda}\frac{ \eta^W}{1-\eta} \tag{ $\norm{\btheta_t} \leq S$ and $\norm{x_t} \leq 1$} 
\end{align*}

We conclude for any $x \in \R^d$
\begin{equation*}
    | x^{\top} (\btheta_t - \bar{\btheta}_t) | \leq \norm{x} \left (\sqrt{\frac{d}{\lambda (1-\eta)}} \sum_{s = t - W}^{t-1} \norm{ \btheta_{s} - \btheta_{s+1}} + \frac{2 S}{\lambda}\frac{ \eta^W}{1-\eta} \right)
\end{equation*}
which proves that
\begin{equation*}
    \norm{\btheta_t - \bar{\btheta}_t} \leq 
    \sqrt{\frac{d}{\lambda (1-\eta)}} \sum_{s = t - W}^{t-1} \norm{ \btheta_{s} - \btheta_{s+1}} + \frac{2 S}{\lambda}\frac{ \eta^W}{1-\eta}
\end{equation*}
Comparing to the bound on $\norm{\btheta_t - \bar{\btheta}_t}$ in the proof of~\citet{russac2019weighted}, there is an extra factor $\sqrt{\frac{d}{\lambda (1-\eta)}}$ that multiplies the local non-stationarity term $\sum_{s = t - W}^{t-1} \norm{ \btheta_{s} - \btheta_{s+1}}$. This extra factor will consequently multiply the variation budget term in the final regret as stated in the following proposition:

\begin{proposition}~\label{prop: D-LinUCB regret} Under the assumption that $\sum_{t=1}^{K-1} \norm{\btheta_{t} - \btheta_{t+1}} \leq \Delta$, for any $\delta \in (0, 1)$, if we set $\beta = \sqrt{\lambda} S + \sigma \sqrt{2 \log(1/\delta) + d \log(1 + \frac{1}{\lambda d (1-\gamma)})}$ in the algorithm 1 \textsc{D-LinUCb} of~\citet{russac2019weighted}, then with probability $1-\delta$, for any $W >0$ the dynamic regret of \textsc{D-LinUCb} is at most 
\begin{align*}
\mathcal{O}\left(  \sqrt{\frac{d}{\lambda (1-\eta)}} \Delta W + \frac{S}{\lambda} \frac{\eta^W}{1-\eta}K + \beta \sqrt{d K} \sqrt{K \log(1/\eta) + \log(1 + \frac{1}{d\lambda (1-\eta)})} \right)
\end{align*}
\end{proposition}

\begin{proposition} Under the same assumption as~\ref{prop: D-LinUCB regret}
If we set $\log(1 / \eta) = d^{-1/4} \Delta^{1/2} K^{-1/2}$, $W = \frac{\log \left( K / (1-\eta) \right)}{\log(1/\eta)}$ and $\lambda =1$; for any $\delta \in (0, 1)$; we have that with probability $1-\delta$, the dynamic regret of \textsc{D-linUCB} is at most $\widetilde{\mathcal{O}}(d^{7/8} \Delta^{1/4} K^{3/4})$.
\end{proposition}
\begin{proof}
With the choice $\log(1 / \eta) = d^{-1/4} \Delta^{1/2} K^{-1/2}$ and $W = \frac{\log \left( K / (1-\eta) \right)}{\log(1/\eta)}$; we have  $\frac{\eta^W}{1-\eta} K = 1$, $\eta = \exp(- \left( \frac{\Delta}{ K} \right)^{1/2}) \substack{\sim \\ K \rightarrow \infty} 1 - d^{-1/4} \Delta^{1/2} K^{-1/2}$  so that $\sqrt{\frac{d}{\lambda (1-\eta)}} \Delta W \sim \sqrt{d} \Delta \log(K / (1-\eta)) \left ( d^{1/4} \Delta^{-1/2} K^{1/2} \right)^{3/2} = \tilde{\calO}( d^{7/8} \Delta^{1/4} K^{3/4})$ and 
$ \beta \sqrt{dK} \sqrt{K \log(1/\eta) + \log(1 + \frac{1}{d\lambda (1-\eta)})}  = \tilde{\calO} (d K \left (d^{-1/4} \Delta^{1/2} K^{-1/2} \right )^{1/2} )= \tilde{\calO}( d^{7/8} \Delta^{1/4} K^{3/4})$.
\end{proof}

\section{Regret Analysis of \textsc{opt-wlsvi} and Proof Outline}

\subsection{Single Step Error Decomposition}
In this section, we analyse the one-step error decomposition of the difference between the estimates $Q_{t,h}$ and $Q^\pi_{t,h}$ of a given policy $\pi$. To do that, we use the weighted MDP $(\calS, \calA, \bar{P}, \bar{r})$ to isolate the bias term. The decomposition contains four parts: the reward bias and variance, the transition bias and variance, and the difference in value functions at step $h+1$. It can be written as:
\begin{align*}
    \bphi(s, a)^\top \w_{t, h} - Q^\pi_{t,h}(s, a) =
    & \underbrace{(\bar{r}_{t,h} - r_{t,h})(s, a)}_{\text{reward bias}} +  \underbrace{(\widehat{r}_{t,h} - \bar{r}_{t,h})(s, a)}_{\text{reward variance}} \\
&~+ \underbrace{[(\bar{P}_{t,h} - P_{t,h})V^\pi_{t, h+1}](s, a)}_{\text{transition bias}} +  \underbrace{[(\widehat{P}_{t,h} - \bar{P}_{t,h})V_{t,h}](s, a)}_{\text{transition variance}} \\
&~+\underbrace{[\bar{P}_{t, h}(V_{t, h+1} - V^\pi_{t, h+1})](s, a)}_{\text{difference in value functions of next step}}.
\end{align*}
The reward and transition bias terms are controlled by Lemma~\ref{lemma: bias} using the fact that $\| V^\pi_{t,h}\|_\infty \leq H$. The difference in value-functions at step $h+1$ can be rewritten as $ [P_{t,h}(V_{t,h+1} - V^\pi_{t, h+1})](s, a) + [(\bar{P}_{t,h} - P_{t, h})(V_{t,h+1} - V^\pi_{t, h+1})](s, a)$. We control the second term by applying again Lemma~\ref{lemma: bias} since $\| V_{t,h+1} - V^\pi_{t, h+1}\|_\infty \leq H$.

It remains now the two variance terms. The reward variance is easy to control and it reduces simply to the bias due to the regularization as we assume that $r$ is a deterministic function. Note that the assumption of deterministic reward is not a limiting assumption since the contribution of a stochastic reward in the final regret has lower order term than the contribution of a stochastic transition. We have, using the Cauchy-Shwartz inequality and $\norm{\tlSigma_{t,h}^{-1}} \leq \frac{1}{\lambda \eta^{-2(t-1)}}$: 

\begin{align*}
    |(\widehat{r}_{t,h} - \bar{r}_{t,h})(s, a)| &= \lambda \eta^{- (t-1)} | \bphi(s, a)^\top \Sigma_{t,h}^{-1} \btheta_{t, h} | \\ 
    & \leq \sqrt{d \lambda} \norm{\bphi(s, a)}_{\Sigma_{t,h}^{-1} \tlSigma_{t,h} \Sigma_{t,h}^{-1}}.
\end{align*}
Controlloing the transition variance is more involved, and we differ the analysis to the next section.
If we define $\texttt{bias} \triangleq \texttt{bias}_{r} + \texttt{bias}_{P}$ the total non-stationarity bias of the MDP, we can summarize the one-step analysis as follows: 
\begin{align} \label{eq: single-step decomp}
     \bphi(s, a)^\top \w_{t, h} - Q^\pi_{t,h}(s, a)  \leq 
     & 2 H \texttt{bias}(t, h) + 
     [P_{t, h}(V_{t, h+1} - V^\pi_{t, h+1})](s, a) \\
    &~+\sqrt{d \lambda} \norm{\bphi(s, a)}_{\Sigma_{t,h}^{-1} \tlSigma_{t,h} \Sigma_{t,h}^{-1}} +  [(\widehat{P}_{t,h} - \bar{P}_{t,h})V_{t,h}](s, a) \nonumber
\end{align}

\subsection{High Probability Bound on the Transition Variance}
In this section, we will establish a high probability bound on the term $(\widehat{P}_{t,h} - \bar{P}_{t,h})V_{t,h}$. From the definitions of $\widehat{P}$ and $\bar{P}$ and the Cauchy-Schwartz inequality, we have
\begin{align*}
& [(\widehat{P}_{t,h} - \bar{P}_{t,h})V_{t,h}](s, a) \leq
\left( \norm{\sum_{\tau=1}^{t-1}  \eta^{-\tau}\bphi_{\tau, h} \epsilon_{\tau, h}}_{\tlSigma_{t,h}^{-1}} + H \sqrt{d \lambda } \right )  \norm{\bphi(s, a)}_{\Sigma_{t,h}^{-1} \tlSigma_{t,h} \Sigma_{t,h}^{-1}},
\end{align*}
where $\epsilon_{\tau, h} = V_{t, h+1}(s_{\tau, h+1}) - [P_{t,h}V_{t, h+1}](s_{\tau, h}, a_{\tau, h})$. If $V_{t, h+1}$ was a fixed function, $\epsilon_{\tau, h}$ would be zero-mean conditioned on the history of transitions up to step $h$ at episode $\tau$ and we would use the concentration of weighted self-normalized processes~\citep{russac2019weighted} to get a high probability bound on $\norm{\sum_{\tau=1}^{t-1}  \eta^{-\tau}\bphi_{\tau, h} \epsilon_{\tau, h}}_{\tlSigma_{t,h}^{-1}}$. However, as $V_{t, h+1}$ is estimated from past transitions and thus depends on the latter in a non-trivial way, we will show a concentration bound that holds uniformly for all possible value functions generated by the algorithm. 
We proceed first by establishing the boundness of iterates in the next Lemma.

\begin{lemma}[Boundness of iterates] \label{lem: iterates boundness}
For any $(t, h) \in [K] \times [H]$, the weight $\w_{t,h}$ and the matrix $\Sigma_t^{-1} \tlSigma_t \Sigma_t^{-1}$ in Algorithm \ref{algo:OPT-WLSVI} satisfies:
\begin{equation*}
\norm{\w_{t,h}} \leq 2 H \sqrt{  \frac{d (1 - \eta^{t-1})}{\lambda (1-\eta)}} \text{ and } \norm{\Sigma_t^{-1} \tlSigma_t \Sigma_t^{-1}} \leq \frac{1}{\lambda}
\end{equation*}
\end{lemma}

Any value function estimate produced by Algorithm~\ref{algo:OPT-WLSVI} could be written in the following form  
\begin{equation*}
    V^{\w, \A}(\cdot) = \min \{ \max_{a \in \calA} \{ \w^\top \bphi(\cdot, a) + \sqrt{\bphi(\cdot, a)^\top \A \bphi(\cdot, a)} \} , H\}
\end{equation*}
where $\w \in \R^d$ and $\A \in \R^{d \times d}$ is a symmetric definite positive matrix that are in 
\begin{equation*}
  \calG = \left \{ \w, \A: \norm{\w} \leq 2 H \sqrt{  \frac{d}{\lambda (1-\eta)}} \text{ and } \fnorm{\A} \leq \frac{\sqrt{d} \beta^2}{\lambda} \right \}  
\end{equation*}

The $\epsilon$-covering number of $\calG$, identified as Euclidean ball in $\R^{d+d^2}$ of radius $2 H \sqrt{  \frac{d}{\lambda (1-\eta)}} + \frac{\beta^2 \sqrt{d}}{\lambda}$, is bounded by $\left( 3\left ( 2 H \sqrt{  \frac{d}{\lambda (1-\eta)}} + \frac{\beta^2 \sqrt{d}}{\lambda} \right)/ \epsilon \right)^{d+d^2}$. The latter number is exponential in the dimension $d$ but only the square root of its logarithm, which is linear in $d$, will contribute to the bound as we will show next. 

By applying the concentration of weighted self-normalized processes~\citep{russac2019weighted} and using a union bound argument over an $\epsilon$-net of $\calG$ with an appropriate value of $\epsilon$, we obtain the desired high probability bound stated in the following Lemma

\begin{lemma} \label{lem: transiton concentration}
For any $\delta \in (0, 1)$, with probability at least $1 - \delta/2$, we have for all $(t, h) \in [K] \times [H]$,
\begin{align*}
\norm{\sum_{\tau=1}^{t-1}  \eta^{-\tau}\bphi_{\tau, h} \epsilon_{\tau, h}}_{\tlSigma_{t,h}^{-1}} \leq C d H  \sqrt{ \log \left(\frac{d H \beta}{\lambda (1 - \eta)}  \cdot \frac{2}{\delta}\right)}
\end{align*}
where $C >0$ is an absolute constant. 
\end{lemma}

Finally, by combining the single error decomposition in Equation~\eqref{eq: single-step decomp} and the transition concentration in Lemma~\ref{lem: transiton concentration} with an appropriate choice of $\beta$, we obtain the following high probability single-step bound.
\begin{lemma}[Key lemma]~\label{lemma: key lemma}
There exists an absolute value $c$ such that $\beta = c d H \sqrt{\imath}$ where $\imath = \log \left( \frac{2 d H}{(1-\eta) \delta} \right)$, $\lambda =1$ and for any fixed policy $\pi$, we have with probability at least $1-\delta/2$ for all $(s, a, h, t) \in \calS \times \calA \times [H] \times [K]$,
\begin{align*}
    & \Big |\bphi(s, a)^\top \w_{t, h} - Q^\pi_{t,h}(s, a) - [P_{t,h}(V_{t,h+1} - V^\pi_{t, h+1})](s, a) \Big | \\
    & \quad \leq  2 H \texttt{bias}(t, h) + \beta
     \norm{\bphi(s, a)}_{\Sigma_{t,h}^{-1} \tlSigma_{t,h} \Sigma_{t,h}^{-1}}.
\end{align*}

\end{lemma}

\subsection{Optimism}
Now, we show that the true value functions can be upper bounded by the value functions
computed by \textsc{opt-wlsvi} plus a bias term. In fact, unlike the stationary case, we act optimistically with respect to the weighted average MDP. To prove this, we use the key Lemma~\ref{lemma: key lemma} in the previous section and we proceed by induction argument over steps $h \in [H]$. 

\begin{lemma}[Optimism] \label{lem: optimism}
For all $(s, a, t, h) \in \calS \times \calA \times [K] \times [H]$, we have with probability at least $1 - \delta / 2$
\begin{align}
    Q_{t,h}(s, a) + 2 H \sum_{h' =h}^H \texttt{bias}(t, h) \geq Q_{t, h}^\star(s, a)
\end{align}
\end{lemma}

\subsection{Final Regret Analysis}
Now, having the results provided in previous sections at hand, we turn to proving the regret bound of our algorithm. Let $\pi_t$  the policy executed by the algorithm in step $h$ for $H$ steps to reach the end of the episode. If we define $\delta_{t,h} \triangleq  V_{t,h}(s_{t, 1}) - V^{\pi_t}_{t,h}(s_{t, h})$, a straightforward application of Lemma~\ref{lem: optimism} is that the regret is upper bounded by the sum of $\delta_{t,h}$ and bias terms with probability at least $1-\delta/2$ i.e
\begin{align} \label{eq: opt bound}
    \textsc{Regret}(K) \underbrace{\leq}_{\text{by optimism}} \sum_{t=1}^K \delta_{t,h} + 2 H  \sum_{t=1}^K \sum_{h=1}^H \texttt{bias}(t, h).
\end{align}
The policy $\pi_t$ is the greedy policy with respect to $Q_{t,h}$, and $a_{t,h} = \pi_t(s_t, h) = \arg \max_{a \in \calA} Q_{t,h}(s_{t,h}, a)$. Therefore, we have $\delta_{t,h} = Q_{t,h}(s_{t,h}, a_{t,h}) - Q^{\pi_t}_{t,h}(s_{t,h}, a_{t,h})$. Using the definition of $Q_{t,h}$ and the key Lemma~\ref{lemma: key lemma}, we obtain with probability at least $1-\delta /2$,
\begin{align*}
    \delta_{t,h}
     & \leq [P_{t,h}(V_{t,h+1} - V^{\pi_t}_{t, h+1})](s_{t,h}, a_{t,h})  + 2 \beta
     \norm{\bphi_{t,h}}_{\Sigma_{t,h}^{-1} \tlSigma_{t,h} \Sigma_{t,h}^{-1}} + 2 H \texttt{bias}(t, h)\\
     & = \delta_{t, h+1} + \xi_{t,h+1} + 2 \beta
     \norm{\bphi_{t,h}}_{\Sigma_{t,h}^{-1} \tlSigma_{t,h} \Sigma_{t,h}^{-1}} 
     + 2 H \texttt{bias}(t, h)
\end{align*}
where we define $\xi_{t,h+1} = [P_{t,h}(V_{t,h+1} - V^{\pi_t}_{t, h+1})](s_{t,h}, a_{t,h}) - (V_{t,h+1} - V^{\pi_t}_{t, h+1})(s_{t, h+1})$. Unrolling the last inequality $H$ times, we obtain 
\begin{align} \label{eq: delta bound}
    \delta_{t,1}  \leq \sum_{h=1}^H \xi_{t,h} + 2 \beta \sum_{h=1}^H \norm{\bphi_{t,h}}_{\Sigma_{t,h}^{-1} \tlSigma_{t,h} \Sigma_{t,h}^{-1}} + 2 H \sum_{h=1}^H \texttt{bias}(t, h)
\end{align}

Hence, by combining Equations~\eqref{eq: opt bound} and~\eqref{eq: delta bound}, we obtain with probability at least $1-\delta /2$, 
\begin{align} \label{eq: regret bound with bonus}
    \textsc{Regret}(K) & \leq \underbrace{\sum_{t=1}^K\sum_{h=1}^H \xi_{t,h} }_{(A)} + 2 \beta \underbrace{\sum_{t=1}^K \sum_{h=1}^H \norm{\bphi_{t,h}}_{\Sigma_{t,h}^{-1} \tlSigma_{t,h} \Sigma_{t,h}^{-1}}}_{(B)} \nonumber \\
    & \quad + 4 H \underbrace{\sum_{t=1}^K \sum_{h=1}^H \texttt{bias}(t, h)}_{(C)}.
\end{align}
Now, we proceed to upper bound the different terms in the RHS of Equation~\eqref{eq: regret bound with bonus}.
\paragraph{\underline{Term (A):}}
The computation of $V_{t, h}$ is independent from $(s_{t,h}, a_{t,h})$, therefore, $\{ \xi_{t,h} \}$ is $2 H$-bounded martingale difference sequence. Therefore, by Azuma-Hoeffding, we have for all $t>0$, $P\left(\left | \sum_{t=1}^K\sum_{h=1}^H \xi_{t,h} \right | \geq t \right) \leq 2 \exp \left ( \frac{- t^2}{16 H^3 K }\right)$. Then, $P\left(\left | \sum_{t=1}^K\sum_{h=1}^H \xi_{t,h} \right | \geq 4 \sqrt{H^3K \log(4/\delta)} \right) \leq \delta /2$.
Therefore, with probability at least $1-\delta/2$, we have
\begin{align} \label{eq: (A) bound}
    \left | \sum_{t=1}^K\sum_{h=1}^H \xi_{t,h} \right | \leq \mathcal{O}(H^{3/2} \sqrt{K \imath})
\end{align}
\paragraph{\underline{Term (B):}} By application of Cauchy-Schwartz, we obtain
\begin{align*}
    \sum_{t=1}^K \sum_{h=1}^H \norm{\bphi_{t,h}}_{\Sigma_{t,h}^{-1} \tlSigma_{t,h} \Sigma_{t,h}^{-1}} \leq \sqrt{K} \sum_{h=1}^H \sqrt{\sum_{t=1}^K \norm{\bphi_{t,h}}^2_{\Sigma_{t,h}^{-1} \tlSigma_{t,h} \Sigma_{t,h}^{-1}}}.
\end{align*}
From Lemma~\ref{lem: iterates boundness}, we have $\norm{\Sigma_t^{-1} \tlSigma_t \Sigma_t^{-1}} \leq \frac{1}{\lambda}$, then, $\norm{\bphi_{t,h}}_{\Sigma_{t,h}^{-1} \tlSigma_{t,h} \Sigma_{t,h}^{-1}} \leq \frac{1}{\sqrt{\lambda}} \norm{\bphi_{t,h}} =  \norm{\bphi_{t,h}} \leq 1$. So, we can use the bound on the sum of the squared norm of the features provided in proposition 4 of~\citet{russac2019weighted} to obtain 
\begin{align} \label{eq: (B) bound}
    \sum_{t=1}^K \sum_{h=1}^H \norm{\bphi_{t,h}}_{\Sigma_{t,h}^{-1} \tlSigma_{t,h} \Sigma_{t,h}^{-1}} \leq H \sqrt{K} \sqrt{2 d K \log(1 / \eta) + 2 d \log \left(1  +  \frac{1}{d \lambda (1-\eta)}\right)}.
\end{align}

\paragraph{\underline{Term (C):}} We control the bias term using the MDP variation budget as follows.
\begin{align} \label{eq: (C) term}
    & \sum_{t=1}^K \sum_{h=1}^H \texttt{bias}(t, h)  \leq \frac{4 H K \sqrt{d}}{\lambda}
     \frac{\eta^W}{1-\eta} + \sqrt{\frac{d}{\lambda (1-\eta)}} \cdot  \nonumber \\ 
    & \sum_{t=1}^K \sum_{h=1}^H \sum_{s = t - W}^{t-1} \| \btheta_{s, h} - \btheta_{s+1, h}\| + \left \| \bmu_{s,h}(\calS) - \bmu_{s+1, h}(\calS) \right\| \nonumber \\
     & \leq \frac{4 H K \sqrt{d}}{\lambda}
     \frac{\eta^W}{1-\eta} + \sqrt{\frac{d}{\lambda (1-\eta)}} W \Delta.
\end{align}
Finally, the desired regret bound in Theorem~\ref{theo: regret_bound} is obtained by combining Equations~\eqref{eq: regret bound with bonus},~\eqref{eq: (A) bound},~\eqref{eq: (B) bound} and~\eqref{eq: (C) term}.

\section{Missing Proofs of Regret Analysis of \textsc{opt-wlsvi}}

\subsection{Linearity of $Q$-values: Lemma~\ref{lem: linear Q}}

\begin{proof} The definition of non-stationary linear MDP from Assumption~\ref{assumption:non_stat_linear} together with the Bellman equation gives:
\begin{align*}
    Q^\pi_{t,h} & = r_{t,h}(s, a) + [\P_{t,h} V^\pi_{t, h+1}](s, a) \\
    & = \bphi(s, a)^\top \btheta_{t,h} + 
    \int_{s'} \bphi(s, a)^\top V^\pi_{t, h+1}(s') d\bmu_{t,h}(s') \\
    & = \bphi(s, a)^\top \left( \btheta_{t,h} + \int_{s'} V^\pi_{t, h+1}(s') d\bmu_{t,h}(s') \right)
\end{align*}
We define $\w^\pi_{t,h}$ to be the term inside the parentheses.
\end{proof}

\subsection{Non-Stationarity Bias}
\subsubsection{Proof of Lemma~\ref{lemma: bias}}
\textbf{\underline{Reward Bias:}}
\begin{align*}
    & |r_{t, h}(s, a) - \bar{r}_{t, h}(s, a)| \\
    & \leq \left | \bphi(s, a)^\top \left( \btheta_{t, h} - \Sigma_{t, h}^{-1} 
    \left( \sum_{\tau = 1}^{t-1} \eta^{-\tau} \bphi_{\tau, h} \bphi_{\tau, h}^\top \btheta_{\tau, h} + \lambda \eta^{-(t-1)} \btheta_{t, h} \right) \right) \right  | \\
    & = \left |    \phi(s,a )^\top \sum_{\tau = 1}^{t-1} \Sigma_{t, h}^{-1} \eta^{-\tau} \bphi_{\tau, h} \bphi_{\tau, h}^\top (\btheta_{t, h} - \btheta_{\tau, h}) \right | \\
    & \leq \underbrace{\left |   \phi(s,a )^\top \sum_{\tau = t - W}^{t-1} \Sigma_{t, h}^{-1} 
     \eta^{-\tau} \bphi_{\tau, h} \bphi_{\tau, h}^\top (\btheta_{t, h} - \btheta_{\tau, h}) \right | }_{(\star)}
     + \underbrace{\left | \phi(s,a )^\top \sum_{\tau = 1}^{t- W - 1} \Sigma_{t, h}^{-1} 
     \eta^{-\tau} \bphi_{\tau, h} \bphi_{\tau, h}^\top (\btheta_{t, h} - \btheta_{\tau, h}) \right | }_{(\star \star)}
\end{align*}

\paragraph{Bound on $(\star)$:}
\begin{align*}
    & \left |   \phi(s,a )^\top \sum_{\tau = t - W}^{t-1} \Sigma_{t, h}^{-1} 
     \eta^{-\tau} \bphi_{\tau, h} \bphi_{\tau, h}^\top (\btheta_{t, h} - \btheta_{\tau, h}) \right | 
     \\
     & = \left |  \sum_{\tau = t - W}^{t-1}\eta^{-\tau}  \phi(s,a )^\top \Sigma_{t, h}^{-1} 
     \bphi_{\tau, h} \bphi_{\tau, h}^\top (\btheta_{t, h} - \btheta_{\tau, h}) \right | \\
     & \leq \sum_{\tau = t - W}^{t-1}\eta^{-\tau}  \left | \phi(s,a )^\top \Sigma_{t, h}^{-1} 
     \bphi_{\tau, h} \right | \cdot \left | \bphi_{\tau, h}^\top (\btheta_{t, h} - \btheta_{\tau, h}) \right | \\
     & \leq \sum_{\tau = t - W}^{t-1}\eta^{-\tau}  \left | \phi(s,a )^\top \Sigma_{t, h}^{-1} 
     \bphi_{\tau, h} \right |  \norm{\bphi_{\tau, h}} \norm{\btheta_{t, h} - \btheta_{\tau, h}} \\
     & \leq \sum_{\tau = t - W}^{t-1}\eta^{-\tau}  \left | \phi(s,a )^\top \Sigma_{t, h}^{-1} 
     \bphi_{\tau, h} \right | \norm{\btheta_{t, h} - \btheta_{\tau, h}} \tag{$\norm{\bphi_{\tau, h}} \leq 1$} \\
     & = \sum_{\tau = t - W}^{t-1}\eta^{-\tau}  \left | \phi(s,a )^\top \Sigma_{t, h}^{-1} 
     \bphi_{\tau, h} \right | \norm{ \sum_{s=\tau}^{t-1}\btheta_{s, h} - \btheta_{s+1, h}} \\
     & \leq \sum_{\tau = t - W}^{t-1}\eta^{-\tau}  \left | \phi(s,a )^\top \Sigma_{t, h}^{-1} 
     \bphi_{\tau, h} \right |\sum_{s=\tau}^{t-1}  \norm{ \btheta_{s, h} - \btheta_{s+1, h}} \\
     & \leq \sum_{s = t - W}^{t-1} \sum_{\tau = t-W}^{s} \eta^{-\tau}  \left | \phi(s,a )^\top \Sigma_{t, h}^{-1} 
     \bphi_{\tau, h} \right |  \norm{ \btheta_{s, h} - \btheta_{s+1, h}} \tag{$\sum_{\tau = t - W}^{t-1} \sum_{s=\tau}^{t-1} =  \sum_{s = t - W}^{t-1} \sum_{\tau = t-W}^{s} $} \\
     & \leq \sum_{s = t - W}^{t-1} \sqrt{ \bigg[ \sum_{\tau = t-W}^{s} \eta^{-\tau} \phi(s, a)^\top \Sigma_{t, h}^{-1}  \phi(s, a) \bigg]  \cdot \biggl [ \sum_{\tau = t-W}^{s}  \eta^{-\tau} \bphi_{\tau, h}^\top \Sigma_{t, h}^{-1} \bphi_{\tau, h}\bigg] }  \\
     & \quad \quad \cdot \norm{ \btheta_{s, h} - \btheta_{s+1, h}} \tag{Cauchy-Schwartz}\\
     & \leq \sum_{s = t - W}^{t-1} \sqrt{ \bigg[ \sum_{\tau = t-W}^{s} \eta^{-\tau} \phi(s, a)^\top \Sigma_{t, h}^{-1}  \phi(s, a) \bigg]  \cdot \sqrt{d} } \cdot  \norm{ \btheta_{s, h} - \btheta_{s+1, h}} \tag{by Lemma~\ref{lem:basic_ineq}} \\
     &  \leq \sum_{s = t - W}^{t-1} \sqrt{d} \sum_{s = t - W}^{t-1} \sqrt{ \frac{\sum_{\tau = t-W}^{t-1} \eta^{-\tau}}{ \lambda \eta^{-(t-1)}}} 
      \norm{ \btheta_{s, h} - \btheta_{s+1, h}}  \tag{$\norm{\phi(s, a)} \leq 1$ and $\lambda_{\max}(\Sigma_{t, h}^{-1}) \leq \frac{1}{\lambda \eta^{-(t-1)}}$} \\
      & \leq \sqrt{\frac{d}{\lambda (1-\eta)}} \sum_{s = t - W}^{t-1} \norm{ \btheta_{s, h} - \btheta_{s+1, h}}
\end{align*}
\paragraph{Bound on $(\star \star)$:}

\begin{align*}
    & \left | \phi(s,a )^\top \sum_{\tau = 1}^{t- W - 1} \Sigma_{t, h}^{-1} 
     \eta^{-\tau} \bphi_{\tau, h} \bphi_{\tau, h}^\top (\btheta_{t, h} - \btheta_{\tau, h}) \right |
    \\ & \leq \frac{1}{\lambda} \norm{\phi(s, a)} \sum_{\tau = 1}^{t- W - 1} 
     \eta^{t-\tau-1} \| \bphi_{\tau, h} \| \cdot  | \bphi_{\tau, h}^\top (\btheta_{t, h} - \btheta_{\tau, h}) | \tag{$\lambda_{\max}(\Sigma_{t, h}^{-1}) \leq \frac{1}{\lambda \eta^{-(t-1)}}$}\\
     & \leq  \frac{1}{\lambda} \norm{\phi(s, a)} \sum_{\tau = 1}^{t- W - 1} \eta^{t-\tau-1} \| \bphi_{\tau, h}\|^2 \| \btheta_{t, h} - \btheta_{\tau, h}\| \\
     & \leq \frac{2 \sqrt{d}}{\lambda}
     \frac{\eta^W}{1-\eta} \tag{$\phi(s, a) \leq 1$ and $\norm{\btheta_{t,h}} \leq \sqrt{d}$}
\end{align*}


\textbf{\underline{Transition Bias:}}
$\forall f: \calS \rightarrow \R$ such that $\| f\|_{\infty} < \infty$ (real-valued bounded function), similarly to what we have done for the reward function, we obtain
\begin{align*}
    \left | [(\P_{t, h} - \bar{\P}_{t,h})f](s, a) \right | & \leq \left \| \Sigma_{t, h}^{-1} \sum_{\tau = 1}^{t-1} \eta^{-\tau} \bphi_{\tau, h} \bphi_{\tau, h}^\top  \int f(s') (d\bmu_{t,h}(s') - d\bmu_{\tau, h}(s')) \right\| \tag{$\| \bphi(s, a)\| \leq 1$} \\
    & \leq  \sqrt{\frac{d}{\lambda (1-\eta)}} \sum_{s = t - W}^{t-1} \left \|  \int f(s') (d\bmu_{s,h}(s') - d\bmu_{s+1, h}(s')) \right\| \\
     & \quad + \frac{1}{\lambda} \sum_{\tau = 1}^{t- W - 1} \eta^{t-\tau-1} \left \| \int f(s') (d\bmu_{t,h}(s') - d\bmu_{\tau, h}(s')) \right \|
\end{align*}
Furthermore, 
\begin{align*}
    \left \| \int f(s') (d\bmu_{t,h}(s') - d\bmu_{\tau, h}(s')) \right \| & = \sqrt{ \sum_{l=1}^d \left | \int f(s') (d\bmu_{t,h}^{(l)}(s') - d\bmu_{\tau, h}^{(l)}(s'))\right|^2} \\
    & \leq \| f\|_{\infty}  \sqrt{ \sum_{l=1}^d | \bmu_{t,h}^{(l)}(\calS) - \bmu_{\tau, h}^{(l)}(\calS))|^2} \\
    & = \| f\|_{\infty} \| \bmu_{s,h}(\calS) - \bmu_{s+1,h}(\calS)\| \\
    & \leq 2 \sqrt{d} \| f\|_{\infty} 
\end{align*}
Therefore,

\begin{align*}
    \left | [(\P_{t, h} - \bar{\P}_{t,h})f](s, a) \right | & \leq  \| f\|_{\infty} \left (  \sqrt{\frac{d}{\lambda (1-\eta)}} \sum_{s = t - W}^{t-1} \left \| (\bmu_{s,h}(\calS) - \bmu_{s+1, h}(\calS)) \right\| + \frac{2 \sqrt{d} } {\lambda} \frac{\eta^W}{1- \eta} \right )
\end{align*}

\subsection{Single Step Error Decomposition}
We provide here the full derivation of the single-error decomposition. We have for all $(t, h) \in [K] \times [H]$

\begin{align*}
    \bphi(s, a)^\top \w_{t, h} - Q^\pi_{t,h}(s, a) =
    & \quad \underbrace{(\bar{r}_{t,h} - r_{t,h})(s, a)}_{\text{reward bias}} +  \underbrace{(\widehat{r}_{t,h} - \bar{r}_{t,h})(s, a)}_{\text{reward variance}} + \\
& \quad \underbrace{[(\bar{P}_{t,h} - P_{t,h})V^\pi_{t, h+1}](s, a)}_{\text{transition bias}} +  \underbrace{[(\widehat{P}_{t,h} - \bar{P}_{t,h})V_{t,h+1}](s, a)}_{\text{transition variance}} + \\
& \quad \underbrace{[\bar{P}_{t, h}(V_{t, h+1} - V^\pi_{t, h+1})](s, a)}_{\text{difference in value functions of next step}}.
\end{align*}

\textbf{\underline{Reward $\&$ transition bias:}} Thanks to Lemma~\ref{lemma: bias}, we have
\begin{align*}
    |\bar{r}_{t, h}(s, a) - r_{t, h}(s, a)| & \leq \texttt{bias}_{r}(t, h)  , \\
     \Big | [(\bar{P}_{t, h} -  P_{t,h})V^\pi_{t,h+1}](s, a) \Big | & \leq \texttt{bias}_{P}(t, h). \tag{$\|V^\pi_{t,h+1}\|_\infty \leq H$}
\end{align*}

\textbf{\underline{Difference in value functions of next step:}}

\begin{align*}
    [\bar{P}_{t, h}(V_{t, h+1} - V^\pi_{t, h+1})](s, a) & = [P_{t,h}(V_{t,h+1} - V^\pi_{t, h+1})](s, a) + [(\bar{P}_{t,h} - P_{t, h})(V_{t,h+1} - V^\pi_{t, h+1})](s, a) \\
    & \leq [P_{t,h}(V_{t,h+1} - V^\pi_{t, h+1})](s, a) + \texttt{bias}_{P}(s, a). \tag{$\| V_{t,h+1} - V^\pi_{t, h+1}\|_\infty \leq H$}
\end{align*}

\textbf{\underline{Reward variance:}}  The reward variance here reduces simply to the bias due to the regularization as we assume that $r$ is a deterministic function.

\begin{align*}
    \Big | (\widehat{r}_{t,h} - \bar{r}_{t,h})(s, a) \Big| &= \lambda \eta^{- (t-1)} | \la \bphi(s, a), \Sigma_{t,h}^{-1} \btheta_{t, h}  \ra| \\ 
    & \leq \lambda \eta^{-(t-1)} \norm{\bphi(s, a)}_{\Sigma_{t,h}^{-1} \tlSigma_{t,h} \Sigma_{t,h}^{-1}} \norm{\Sigma_{t,h}^{-1} \btheta_{t,h}}_{\Sigma_{t,h} \tlSigma_{t,h}^{-1} \Sigma_{t,h}} \\
    & = \lambda \eta^{-(t-1)} \norm{\bphi(s, a)}_{\Sigma_{t,h}^{-1} \tlSigma_{t,h} \Sigma_{t,h}^{-1}} \norm{\btheta_{t,h}}_{\tlSigma_{t,h}^{-1}} \\
    & \leq  \lambda \eta^{-(t-1)} \norm{\bphi(s, a)}_{\Sigma_{t,h}^{-1} \tlSigma_{t,h} \Sigma_{t,h}^{-1}} \sqrt{\norm{\tlSigma_{t,h}^{-1}}} \norm{\btheta_{t,h}} \\
    & \leq \sqrt{d \lambda} \norm{\bphi(s, a)}_{\Sigma_{t,h}^{-1} \tlSigma_{t,h} \Sigma_{t,h}^{-1}}
\end{align*}
The last step follows from $\norm{\btheta_{t,h}} \leq \sqrt{d}$ (Assumption~\ref{assumption:non_stat_linear}) and $\norm{\tlSigma_{t,h}^{-1}} \leq \frac{1}{\lambda \eta^{-2(t-1)}}$.

If we define $\texttt{bias} \triangleq \texttt{bias}_{r} + 2 \cdot  \texttt{bias}_{P}$ the total non-stationarity bias of the MDP, we can summarize the one-step analysis as follows: 
\begin{align*} 
     \bphi(s, a)^\top \w_{t, h} - Q^\pi_{t,h}(s, a) & \leq
     \quad \texttt{bias}(t, h) + 
     [\P_{t, h}(V_{t, h+1} - V^\pi_{t, h+1})](s, a) + \nonumber \\
    & \quad \sqrt{d \lambda} \norm{\bphi(s, a)}_{\Sigma_{t,h}^{-1} \tlSigma_{t,h} \Sigma_{t,h}^{-1}} +  [(\widehat{P}_{t,h} - \bar{P}_{t,h})V_{t,h}](s, a) \nonumber
\end{align*}

\subsection{Boundness of iterates}
We will start with the following elementary lemma:
\begin{lemma}\label{lem:basic_ineq}
Let $\Sigma_t = \sum_{\tau=1}^{t-1} \eta^{-\tau} \bphi_\tau \bphi_{\tau}^\top + \lambda \eta^{-(t-1)} \I$ where $\bphi_\tau \in \R^d$ and $\lambda > 0 , \eta \in (0, 1)$. Then:
\begin{equation*}
    \sum_{\tau=1}^{t-1} \eta^{-\tau}\bphi_\tau^\top \Sigma_t^{-1}\bphi_\tau \leq d
\end{equation*}
\end{lemma}

\begin{proof}
We have $\sum_{\tau=1}^{t-1} \eta^{-\tau} \bphi_{\tau}^{\top} \Sigma_t^{-1}\bphi_{\tau} = \sum_{\tau=1}^{t-1} \text{tr}\left( \eta^{-\tau} \bphi_\tau^\top \Sigma_t^{-1}\bphi_\tau \right) = \text{tr}\left( \Sigma_t^{-1} \sum_{\tau=1}^{t-1}\eta^{-\tau} \bphi_\tau \bphi_\tau^\top \right)$. Given the eigenvalue decomposition $\sum_{\tau=1}^{t-1} \eta^{-\tau} \bphi_\tau  \bphi_\tau = \text{diag}(\lambda_1, \ldots, \lambda_d)^\top$, we have $\Sigma_t = \text{diag}(\lambda_1 + \lambda \eta^{-(t-1)}, \ldots, \lambda_d + \lambda \eta^{-(t-1)})^\top$, and $\text{tr} \left( \Sigma_t^{-1} \sum_{\tau=1}^{t-1}\eta^{-\tau} \bphi_\tau \bphi_\tau^\top\right) = \sum_{i=1}^d \frac{\lambda_j}{\lambda_j + \lambda \eta^{-(t-1)}} \leq d$
\end{proof}

\subsubsection{Proof of Lemma~\ref{lem: iterates boundness}}

\textbf{\underline{Bound on $\norm{\w_{t,h}}$:}}
For any vector $v \in \R^d$, we have
\begin{align*}
|v^\top \w_{t,h}| & = \left |v^\top \Sigma_{t,h}^{-1} \sum_{\tau=1}^{t-1} \eta^{-\tau} \bphi_{\tau,h} [r_{\tau,h}+ \max_a Q_{\tau, h+1}(s^{\tau, h+1}, a)] \right|\\
& \le \sum_{\tau = 1}^{t-1} \eta^{-\tau} |v^\top \Sigma_{t,h}^{-1} \bphi_{\tau,h} | \cdot 2H 
\le \sqrt{ \bigg[ \sum_{\tau = 1}^{t-1}  \eta^{-\tau} v^\top\Sigma_{t,h}^{-1} v\bigg]  \cdot \biggl [ \sum_{\tau = 1}^{t-1}  \eta^{-\tau} \bphi_{\tau,h}^\top \Sigma_{t,h}^{-1} \bphi_{\tau,h}\bigg] } \cdot 2H\\
& \le 2H \norm{v}\sqrt{ \frac{ \sum_{\tau = 1}^{t-1}  \eta^{-\tau}}{ \lambda \eta^{-(t-1)}} \cdot d}  = 2 H \norm{v} \sqrt{  \frac{d (1 - \eta^{t-1})}{\lambda (1-\eta)}}
\end{align*}
where the third inequality is due to Lemma \ref{lem:basic_ineq} and the fact that the eigenvalues of $\Sigma_{t,h}^{-1}$ are upper bounded by $\frac{1}{\lambda \eta^{-(t-1)}}$. The remainder of the proof follows from the fact that 
$\norm{\w_{t,h}} = \max_{v:\norm{v} = 1} |v^\top \w_{t,h}|$.

\textbf{\underline{Bound on $\norm{\Sigma_t^{-1} \tlSigma_t \Sigma_t^{-1}}$:}}
\begin{equation}
    \tlSigma_t = \sum_{\tau=1}^{t-1} \eta^{-2\tau} \bphi_\tau \bphi_{\tau}^\top + \lambda \eta^{-2(t-1)} \I \leq  \eta^{-(t-1)} \sum_{\tau=1}^{t-1} \eta^{-\tau} \bphi_\tau \bphi_{\tau}^\top + \lambda \eta^{-2(t-1)} \I = \eta^{-(t-1)} \Sigma_t
\end{equation}
Hence,
\begin{equation*}
    \Sigma_t^{-1} \tlSigma_t \Sigma_t^{-1} \leq \eta^{-(t-1)} \Sigma_t^{-1} \Sigma_t \Sigma_t^{-1} =  \eta^{-(t-1)} \Sigma_t^{-1}
\end{equation*}
and 
\begin{equation*}
    \norm{\Sigma_t^{-1} \tlSigma_t \Sigma_t^{-1}} \leq \eta^{-(t-1)} \norm{\Sigma_t^{-1}} \leq \eta^{-(t-1)} \frac{1}{\lambda \eta^{-(t-1)}} = \frac{1}{\lambda}
\end{equation*}

\subsection{Transition Concentration}

\begin{align}
    & \left |[(\widehat{P}_{t,h} - \bar{P}_{t,h})V_{t,h}](s, a) \right| \nonumber \\
    & \leq \Big | \bphi(s, a)^\top \Big( \Sigma_{t,h}^{-1} \sum_{\tau =1}^{k-1} \eta^{-\tau} \bphi_{\tau, h} (V_{t, h+1}(s_{\tau, h+1}) - [\P_{t,h} V_{t, h+1}](s_{\tau, h}, a_{\tau, h}) ) \nonumber \\
    &~- \lambda \eta^{-(t-1)}  \Sigma_{t,h}^{-1} \bmu_{t, h}V_{t, h+1} \Big) \Big |  \nonumber \\
    & \leq \norm{\bphi(s, a)}_{\Sigma_{t,h}^{-1} \tlSigma_{t,h} \Sigma_{t,h}^{-1}} \Big( \norm{\sum_{\tau=1}^{t-1}  \eta^{-\tau}\bphi_{\tau, h} \left( V_{t, h+1}(s_{\tau, h+1}) - [\P_{t,h}V_{t, h+1}](s_{\tau, h}, a_{\tau, h}) \right)}_{\tlSigma_{t,h}^{-1}} \nonumber \\
    & \quad + \lambda \eta^{-(t-1)} \sqrt{\norm{\tlSigma_{t,h}^{-1}}} \norm{\bmu_{t,h}(\calS)} \norm{V_{t,h+1}}_\infty \Big) \tag{Cauchy-Schwarz} \nonumber \\
    & \leq \norm{\bphi(s, a)}_{\Sigma_{t,h}^{-1} \tlSigma_{t,h} \Sigma_{t,h}^{-1}} \Big( \norm{\sum_{\tau=1}^{t-1}  \eta^{-\tau}\bphi_{\tau, h} \left( V_{t, h+1}(s_{\tau, h+1}) - [\P_{t,h}V_{t, h+1}](s_{\tau, h}, a_{\tau, h}) \right)}_{\tlSigma_{t,h}^{-1}} + H \sqrt{d \lambda } \Big ) \label{eq: transition decomp}
\end{align}
The last step follows from $\norm{\bmu_{t,h}(\calS)} \leq \sqrt{d}$ (Assumption~\ref{assumption:non_stat_linear}) and $\norm{\tlSigma_{t,h}^{-1}} \leq \frac{1}{\lambda \eta^{-2(t-1)}}$.

Let us now consider the following function form:
\begin{equation} \label{eq: V form}
    V^{\w, \A}(\cdot) = \min \{ \max_{a \in \calA} \{ \w^\top \bphi(\cdot, a) + \sqrt{\bphi(\cdot, a)^\top \A \bphi(\cdot, a)} \} , H\}
\end{equation}
where $\w \in \R^d$ and $\A \in \R^{d \times d}$ is a symmetric definite positive matrix that are in 
\begin{equation} \label{eq: G set}
  \calG = \left \{ \w, \A: \norm{\w} \leq 2 H \sqrt{  \frac{d}{\lambda (1-\eta)}} \text{ and } \fnorm{\A} \leq \frac{\sqrt{d} \beta^2}{\lambda} \right \}  
\end{equation}

In the technical Lemma~\ref{lemma:union_bound_covering}, we prove a concentration bound that holds uniformly for any function on the form $V^{\w, \A}$ where $\w, \A \in \calG$. The statement and full proof of this lemma is defered to section~\ref{sec: technical lems} of the appendix. As a corollary of Lemma~\ref{lemma:union_bound_covering}, we can prove the concentration of the transition as follows.

 For any $\tau > 0, h \in [H]$, let $\calF_{\tau, h}$ be the $\sigma$-field generated by all the random variables until episode $\tau$, step $h$. $\{ s_{\tau, h} \}$ defines a stochastic process on state space $\calS$ with corresponding filtration $\{ \calF_{\tau, h}\}$. We have 
\begin{align*}
    V_{\tau, h+1}(\cdot) & = \max_{a \in \calA} \{ \min\{ \w_{t,h+1}^\top \bphi(\cdot, a) + \beta_t  [\bphi(\cdot, a)^\top \Sigma_{t, h}^{-1} \tlSigma_{t, h} \Sigma_{t, h}^{-1} \bphi(\cdot, a)]^{1/2}, H\} \} \\
    & = V^{\w, \A}(\cdot)
\end{align*}
where $\w = \w_{t,h+1}$ and $\A = \beta_t^2 \Sigma_{t, h}^{-1} \tlSigma_{t, h} \Sigma_{t, h}^{-1}$. We have $\|\w\| \leq 2 H \sqrt{  \frac{d (1 - \eta^{t-1})}{\lambda (1-\eta)}} $ and $\| \A\|_{F} \leq  \sqrt{d} \beta_t^2 \| A\| \leq \frac{\sqrt{d} \beta_t^2}{\lambda}$ by Lemma~\ref{lem: iterates boundness}. Therefore $(\w, \A) \in \calG$ and we can apply Lemma~\ref{lemma:union_bound_covering}: we have with probability at least $1-\delta/2$
\begin{align} 
    & \norm{\sum_{\tau=1}^{t-1}  \eta^{-\tau}\bphi_{\tau, h} \left( V_{t, h+1}(s_{\tau, h+1}) - [\P_{t,h}V_{t, h+1}](s_{\tau, h}, a_{\tau, h}) \right)}_{\tlSigma_{t,h}^{-1}} \nonumber \\
    & \leq C
    d H \sqrt{ \log \left( d H \beta \frac{1}{\lambda (1 - \eta)} \right) + \log (2/\delta)}
    \label{eq: transiton concentration}
\end{align}

Combining Equation~\eqref{eq: transition decomp} and \eqref{eq: transiton concentration}, we obtain that with probability at least $1-\delta/2$: 

\begin{align*}
    & \left |[(\widehat{P}_{t,h} - \bar{P}_{t,h})V_{t,h}](s, a) \right| \\
    & \leq \norm{\bphi(s, a)}_{\Sigma_{t,h}^{-1} \tlSigma_{t,h} \Sigma_{t,h}^{-1}} \Big( C
    d H \sqrt{ \log \left( d H \beta \frac{1}{\lambda (1 - \eta)} \right) + \log (2/\delta)} + H \sqrt{d \lambda } \Big )
\end{align*}

\subsection{Single-Step High Probability Upper Bound}
\subsubsection{Proof of Lemma~\ref{lemma: key lemma}}
We have shown so far: 
\begin{align*}
     \bphi(s, a)^\top \w_{t, h} & - Q^\pi_{t,h}(s, a)  \leq
     \quad \texttt{bias}(t, h) + 
     [\P_{t, h}(V_{t, h+1} - V^\pi_{t, h+1})](s, a)  \nonumber \\
    & + \norm{\bphi(s, a)}_{\Sigma_{t,h}^{-1} \tlSigma_{t,h} \Sigma_{t,h}^{-1}} \Big( C
    d H \sqrt{ \log \left( d H \beta \frac{1}{\lambda (1 - \eta)} \right) + \log (2/\delta)} + H \sqrt{d \lambda } + \sqrt{d \lambda } \Big )
\end{align*}
so there exists an absolute constant $C' > 0$ such that
\begin{align*}
     \bphi(s, a)^\top \w_{t, h} & - Q^\pi_{t,h}(s, a) \leq
     \texttt{bias}(t, h) + 
     [\P_{t, h}(V_{t, h+1} - V^\pi_{t, h+1})](s, a)  \nonumber \\
    & + \norm{\bphi(s, a)}_{\Sigma_{t,h}^{-1} \tlSigma_{t,h} \Sigma_{t,h}^{-1}} \Big( C'
    d H \sqrt{\lambda} \sqrt{ \log \left( d H \beta \frac{1}{\lambda (1 - \eta)} \right) + \log (2/\delta)} \Big )
\end{align*}
The missing ingredient to prove our key lemma is the choice of the parameter $\beta$

Now, we would like to find an appropriate choice of $\beta$ such that
\begin{align}
    C' d H  \sqrt{\lambda}\sqrt{ \log \left( d H \beta \frac{1}{\lambda (1 - \eta)} \right) + \log (2/\delta)} \leq \beta 
\end{align}
First we set $\lambda = 1$. A good candidate for $\beta$ is in the form of $\beta = c_\beta d H \sqrt{ \imath}$ where $c_\beta > 1$ is an absolute constant and $\imath =  \log \left( \frac{2 d H}{\delta (1-\eta)}\right)$. With this choice, we obtain:
\begin{align*}
    C' d H  \sqrt{\lambda}\sqrt{ \log \left( d H \beta \frac{1}{\lambda (1 - \eta)} \right) + \log (2/\delta)} & = C' d H  \sqrt{ \log(c_\beta d H \sqrt{\imath}) + \imath} \\
    & \leq C'' d H \sqrt{ \log (c_\beta) + \imath} \tag{$C''>0$ is an absolute constant} \\
    & \leq C'' d H (\sqrt{ \log (c_\beta)}  +\sqrt{\imath} )
\end{align*}

Let $c_\beta > 1$ such that $C'' (\sqrt{ \log (c_\beta)}  + \sqrt{\log(2)} ) \leq \frac{c_\beta}{\sqrt{2}} \sqrt{\log(2)}$. In particular we have necessarily $\frac{c_\beta}{\sqrt{2}} \geq C''$ Therefore, we have: 
\begin{align*}
    C'' (\sqrt{ \log (c_\beta)}  + \sqrt{\imath} ) & = C'' (\sqrt{ \log (c_\beta)}  + \sqrt{\log(2) + (\imath - \log(2) } ) \tag{$\imath \geq \log(2)$} \\
    & \leq C'' (\sqrt{ \log (c_\beta)}  + \sqrt{\log(2)} + \sqrt{(\imath - \log(2) } ) \\
    & \leq \frac{c_\beta}{\sqrt{2}} \sqrt{\log(2)} + C'' \sqrt{(\imath - \log(2) } \\
    & \leq \frac{c_\beta}{\sqrt{2}} (\sqrt{\log(2)} + \sqrt{(\imath - \log(2) }) \\
    & \leq \frac{c_\beta}{\sqrt{2}} \sqrt{2} \sqrt{\log(2) + \imath - \log(2) }  \tag{ $(a + b)^2 \leq 2 (a^2 + b^2) \Rightarrow a+b \leq \sqrt{2 (a^2 + b^2)}$}\\
    & \leq c_\beta \sqrt{\imath} 
\end{align*}

Therefore, with this choice of $c_\beta$ and $\beta = c_\beta d H \sqrt{\imath}$, we obtain that 

\begin{align*}
    \left | \la \bphi(s, a), \w_{t, h} \ra - Q^\pi_{t,h}(s, a) - [\P_{t,h}(V_{t,h+1} - V^\pi_{t, h+1})](s, a) \right | \leq \texttt{bias}(t, h) + \beta
     \norm{\bphi(s, a)}_{\Sigma_{t,h}^{-1} \tlSigma_{t,h} \Sigma_{t,h}^{-1}}.
\end{align*}

\subsection{Optimism}

\subsubsection{Proof of Lemma~\ref{lem: optimism}}
We proceed by induction. By definition, we have $Q_{t,H+1} = Q^\star_{t, H+1} = 0$ and the desired statement trivially holds at step $H+1$.
Now, assume that the statement holds for $h+1$. Consider step $h$. By Lemma~\ref{lemma: key lemma}, we have
\begin{align*}
    & \left | \bphi(s, a)^\top \w_{t, h} - Q^\star_{t,h}(s, a) - [\P_{t,h}(V_{t,h+1} - V^\star_{t, h+1})](s, a) \right | \leq  \textbf{bias}(t, h) + \beta
     \norm{\bphi(s, a)}_{\Sigma_{t,h}^{-1} \tlSigma_{t,h} \Sigma_{t,h}^{-1}}
\end{align*}
Moreover, we have
\begin{align*}
    V^\star_{t, h+1}(s) - V_{t, h+1}(s) & = \max_{a \in \calA} Q^\star_{t, h+1}(s, a) - \max_{a \in \calA} Q_{t, h+1}(s, a) \\
    & \leq \max_{a \in \calA} \left( Q^\star_{t, h+1}(s, a) - Q_{t, h+1}(s, a) \right) \\
    & \leq \sum_{h' = h+1}^H \textbf{bias}(t, h) \tag{ by the induction hypothesis}
\end{align*}
Therefore, we obtain
\begin{align*}
    Q^\star_{t,h}(s, a) & \leq \bphi(s, a)^\top \w_{t, h}  + \beta \norm{\bphi(s, a)}_{\Sigma_{t,h}^{-1} \tlSigma_{t,h} \Sigma_{t,h}^{-1}} + [\P_{t,h} (V^\star_{t, h+1} - V_{t, h+1})](s, a) + \textbf{bias}(t,h)  \\
    & \leq \bphi(s, a)^\top \w_{t, h}  + \beta \norm{\bphi(s, a)}_{\Sigma_{t,h}^{-1} \tlSigma_{t,h} \Sigma_{t,h}^{-1}} + \sum_{h'=h}^H \textbf{bias}(t, h) 
\end{align*}
We have 
\begin{align*}
     Q^\star_{t,h}(s, a) & \leq 
     \bphi(s, a)^\top\w_{t, h} \ra + + \beta \norm{\bphi(s, a)}_{\Sigma_{t,h}^{-1} \tlSigma_{t,h} \Sigma_{t,h}^{-1}}+ \sum_{h'=h}^H \textbf{bias}(t, h) \\
     & = Q_{t,h}(s, a) + \sum_{h'=h}^H \textbf{bias}(t, h)
\end{align*}

\section{Technical Lemmas} \label{sec: technical lems}

\begin{lemma}[Concentration of weighted self-normalized processes \citep{russac2019weighted}] \label{lemma:self_normalized_process}
Let $\{ \epsilon_t\}_{t=1}^{\infty}$ be a real-valued stochastic process with corresponding filtration $\{ \calF\}_{t=1}^\infty$. Let $\epsilon_t \mid \calF_{t-1}$ be zero-mean and $\sigma$-subGaussiasn; i.e $\E[\epsilon_t \mid \calF_{t-1}] =0$ and 
\begin{equation*}
    \forall \lambda \in \R, \quad \E[e^{\lambda \epsilon_t}\mid \calF_{t-1}] \leq e^{\lambda^2 \sigma^2/2}
\end{equation*}
Let $\{ \bphi_t\}_{t=1}^{\infty}$ be a predictable $\R^d$-valued stochastic process (i.e $\bphi_t$ is $\calF_{t-1}$-measurable) and $\{ \omega_t\}_{t=0}^\infty$ be a sequence of predictable and positive weights. Let $\tlSigma_t = \sum_{s=1}^t \omega_s^2 \bphi_s \bphi_s^\top + \mu_t \cdot \I$ where $\{\mu_t \}_{t=1}^\infty$ a deterministic sequence of scalars. Then for any $\delta > 0$, with probability at least $1 - \delta$, we have for all $t \geq 0$:
\begin{equation}
    \norm{\sum_{s = 1}^t \omega_s \bphi_s \epsilon_s }_{\tlSigma_t^{-1}} \leq \sigma \sqrt{2 \log\left( \frac{1}{\delta}\right) + \log \left( \frac{\text{det}(\tlSigma_t)}{\mu_t^d}\right)}
\end{equation}

\end{lemma}

\begin{lemma}[Determinant inequality for the weighted Gram matrix \citet{russac2019weighted}] \label{lemma:determinant}
Let $\{ \lambda_t\}_{t=0}^\infty$ and $\{ \omega_t\}_{t=0}^\infty$ be a deterministic sequence of scalars. Let $\Sigma_t = \sum_{t=1}^t w_s \bphi_s \bphi_s^\top + \lambda_t \cdot \I$ be the weighted Gram matrix. Under the assumption $\forall t, \norm{\bphi_t} \leq 1$, the following holds
\begin{equation}
    \text{det}(\Sigma_t) \leq \left( \lambda_t + \frac{\sum_{s=1}^t \omega_s}{d}\right)^d
\end{equation}

\end{lemma}

\begin{lemma}[Covering Number of Euclidean Ball~\citep{pollard1990empirical}]\label{lemma:covering_number} For any $\epsilon > 0$, the $\epsilon$-covering number of the Euclidean
ball in $\R^d$ with radius $R > 0$ is upper bounded by $(\frac{3 R}{\epsilon})^d$
\end{lemma}

\begin{lemma}[Uniform concentration]\label{lemma:union_bound_covering}
Let $\{ s_t\}_{t=1}^\infty$ be a stochastic process on state state $\calS$ with corresponding filtration $\{ \calF_t\}_{t=0}^\infty$. Let $\{ \bphi_t\}_{t=1}^\infty$ be an $\R^d$-valued stochastic process where $\bphi_t$ is $\calF_{t-1}$-measurable, and $\norm{\bphi} \leq 1$. Let $\tlSigma_t = \sum_{\tau=1}^t \eta^{-2\tau} \bphi_\tau \bphi_{\tau}^\top + \lambda \eta^{-2t} \cdot \I$. Then for any $\epsilon \in (0, 1)$ and $\delta > 0$, with probability at least $1-\delta$, for all $t>0$ and for all $\w, \A \in \calG$ defined in~\eqref{eq: G set},  we have
\begin{align}
    & \norm{\sum_{\tau=1}^t \eta^{-\tau} \bphi_\tau \left( V^{\w, \A}(s_\tau)  - \E \left[ V^{\w, \A}(s_\tau) \mid \calF_{\tau -1}\right] \right)}_{\tlSigma_t^{-1}}
    \nonumber \\
    & \leq \mathcal{O}\left(d H  \sqrt{ \log \left( d H \beta_t \frac{1}{\lambda (1 - \eta)} \right) + \log (1/\delta)} \right)
\end{align}
where $V^{\w, \A}$ is defined in Equation~\eqref{eq: V form}.
\end{lemma}

\begin{proof} Let $t>0$.
For $\w, \A \in \mathcal{O}_t$ and $\tau \in [t]$ define
\begin{equation}
    \epsilon_\tau^{\w, \A} = V^{\w, \A}(s_\tau) - \E \left[ V^{\w, \A}(s_\tau) \mid \calF_{\tau -1} \right]
\end{equation}
Then $\epsilon_\tau^{\w, \A}$ defines a martingale difference sequence with filtration $\calF_\tau$. Moreover, by the definition of $V^{\w, \A}$, each $\epsilon_\tau^{\w, \A}$ is bounded in absolute value by $H$, so that each $\epsilon_\tau^{\w, \A}$ is $H$-subgaussian random variable.

So, by lemma~\ref{lemma:self_normalized_process}, the $\epsilon_\tau^{\w, \A}$ induce a self normalizing process so that for any $\delta > 0$, with probability at least $1 - \delta$, we have for all $t>0$:
\begin{align}
    \norm{\sum_{\tau = 1}^t \eta^{-\tau} \bphi_s \epsilon^{\w, \A}_\tau }_{\tlSigma_t^{-1}} & \leq H \sqrt{2 \log\left( \frac{1}{\delta}\right) + \log \left( \frac{\text{det}(\tlSigma_t)}{(\lambda \eta^{-2t})^{d}}\right)} \\
    & \leq H \sqrt{2 \log\left( \frac{1}{\delta}\right) + d \log \left( 1 +  \frac{1-\eta^{2t}}{\lambda d(1 - \eta^2)}\right)} 
\end{align}
The last step is due to $\text{det}(\tlSigma_t) \leq \left(\lambda \eta^{-2t} + \frac{\eta^{-2t} -1}{d (1 - \eta^{-d})} \right)^d $ by lemma~\ref{lemma:determinant}.

Let $\mathcal{N}_{\epsilon}(\calG)$ be covering number of $\calG$.
So, by union bound, with probability $\delta$. For all $\tilde{\w} ,\tilde{\A}$ in the $\epsilon$-covering of $\calG$ that
\begin{align}
    \norm{\sum_{\tau = 1}^t \eta^{-\tau} \bphi_s \epsilon^{\tilde{\w} ,\tilde{\A}}_\tau }_{\tlSigma_t^{-1}} & 
    \leq H \sqrt{2 \log\left( \frac{\mathcal{N}_\epsilon(\calG)}{\delta}\right) + d \log \left( 1 +  \frac{1-\eta^{2t}}{\lambda d(1 - \eta^2)}\right)} 
\end{align}

For any $(\w, \A) \in \calG$, we choose a specific $(\tilde{\w}, \tilde{\A})$ in the $\epsilon$-covering of $\calG$ such that $\norm{\w - \tilde{\w}} \leq \epsilon$ and $\fnorm{\A - \tilde{\A}} \leq \epsilon$.

\begin{align}
    \norm{\sum_{\tau = 1}^t \eta^{-\tau} \bphi_s \epsilon^{\w ,\A}_\tau }_{\tlSigma_t^{-1}} & \leq \norm{\sum_{\tau = 1}^t \eta^{-\tau} \bphi_s \epsilon^{\tilde{\w} ,\tilde{\A}}_\tau }_{\tlSigma_t^{-1}} + \norm{\sum_{\tau = 1}^t \eta^{-\tau} \bphi_s \left( \epsilon^{\w ,\A}_\tau - \epsilon^{\tilde{\w} ,\tilde{\A}}_\tau \right)}_{\tlSigma_t^{-1}}   \nonumber \\
    & \leq H \sqrt{2 \log\left( \frac{\mathcal{N}_\epsilon(\calG)}{\delta}\right) + d \log \left( 1 +  \frac{1-\eta^{2t}}{\lambda d(1 - \eta^2)}\right)}  \nonumber 
    \\
    & + \norm{\sum_{\tau = 1}^t \eta^{-\tau} \bphi_s \left( \epsilon^{\w ,\A}_\tau - \epsilon^{\tilde{\w} ,\tilde{\A}}_\tau \right)}_{\tlSigma_t^{-1}} \label{eq: bound with covering}
\end{align}
We can bound
\begin{align*}
\norm{\sum_{\tau = 1}^t \eta^{-\tau} \bphi_s \left( \epsilon^{\w ,\A}_\tau - \epsilon^{\tilde{\w} ,\tilde{\A}}_\tau \right)}_{\tlSigma_t^{-1}} & \leq 
\frac{1}{\sqrt{\lambda} \eta^{-t}} 
\norm{\sum_{\tau = 1}^t \eta^{-\tau} \bphi_s \left( \epsilon^{\w ,\A}_\tau - \epsilon^{\tilde{\w} ,\tilde{\A}}_\tau \right)} \\
& \leq \frac{1}{\sqrt{\lambda} \eta^{-t}} \cdot \frac{\eta^{-t} - 1}{1-\eta}  \sup_{\tau} | \epsilon^{\w ,\A}_\tau - \epsilon^{\tilde{\w} ,\tilde{\A}}_\tau| \\
& = 
\frac{1 - \eta^{t}}{\sqrt{\lambda}(1-\eta)} \sup_{\tau} | \epsilon^{\w ,\A}_\tau - \epsilon^{\tilde{\w} ,\tilde{\A}}_\tau| \\
& \leq \frac{2 (1 - \eta^{t})}{\sqrt{\lambda}(1-\eta)} \sup_{\tau} | V^{\w ,\A}(s_\tau) - V^{\tilde{\w} ,\tilde{\A}}(s_\tau)|   
\end{align*}
By the definition of $V^{\w, \A}$, we have
\begin{align*}
   & \sup_{\tau} |  V^{\w ,\A}(s_\tau) - V^{\tilde{\w} ,\tilde{\A}}(s_\tau)| \\
   & \leq \sup_{s, a} \left| \left( \w^\top \bphi(s, a) + \sqrt{\bphi(s, a)^\top \A \bphi(s, a)} \right) - \left( \tilde{\w}^\top \bphi(s, a) + \sqrt{\bphi(s, a)^\top \tilde{\A} \bphi(s, a) }\right) \right | \\
   \\ & \leq \sup_{\bphi \in \R^d: \| \bphi\| \leq 1} \left | \left( \w^\top \bphi + \sqrt{\bphi^\top \A \bphi} \right) - \left( \tilde{\w}^\top \bphi + \sqrt{\bphi^\top \tilde{\A} \bphi }\right)\right | \\
   & \leq \sup_{\bphi \in \R^d: \| \bphi\| \leq 1} |(\w - \tilde{\w})^\top \bphi| + \sup_{\bphi \in \R^d: \| \bphi\| \leq 1} \sqrt{\bphi^\top (\A - \tilde{\A}) \bphi } \\
   & = \norm{ \w - \tilde{\w} }+ \sqrt{\norm{\A - \tilde{\A} }}\\
  & \leq \norm{ \w - \tilde{\w}} + \sqrt{\fnorm{\A - \tilde{\A}}} \leq \epsilon + \sqrt{\epsilon} \leq 2 \sqrt{\epsilon}.
\end{align*}
Therefore,
\begin{align} \label{eq: covering bias}
    \norm{\sum_{\tau = 1}^t \eta^{-\tau} \bphi_s \left( \epsilon^{\w ,\A}_\tau - \epsilon^{\tilde{\w} ,\tilde{\A}}_\tau \right)}_{\tlSigma_t^{-1}} & \leq  
    \frac{4 (1 - \eta^{t})}{\sqrt{\lambda}(1-\eta)} \sqrt{\epsilon}
\end{align}

The $\epsilon$-covering number of $\calG$ as Euclidean ball in $\R^{d+d^2}$ of radius $2 H \sqrt{  \frac{d}{\lambda (1-\eta)}} + \frac{\beta^2 \sqrt{d}}{\lambda}$ is bounded by Lemma~\ref{lemma:covering_number} as $\left( 3\left ( 2 H \sqrt{  \frac{d }{\lambda (1-\eta)}} + \frac{\beta^2 \sqrt{d}}{\lambda} \right)/ \epsilon \right)^{d+d^2}$. Now, combining Equations~\eqref{eq: bound with covering} and~\eqref{eq: covering bias} we obtain:

\begin{align*}
     & \norm{\sum_{\tau = 1}^t \eta^{-\tau} \bphi_s \epsilon^{\w ,\A}_\tau }_{\tlSigma_t^{-1}} \\
     & \leq 2H \sqrt{2 (d^2 + d)\log\left(\frac{3}{\epsilon}\left ( 2 H \sqrt{  \frac{d}{\lambda (1-\eta)}} + \frac{\beta^2 \sqrt{d}}{\lambda} \right)\right) + 2 \log(\frac{1}{\delta}) + d \log \left( 1 +  \frac{1-\eta^{2t}}{\lambda d(1 - \eta^2)}\right)} \\
     &
     + \frac{4 \sqrt{\epsilon} (1 - \eta^{t})}{\sqrt{\lambda}(1-\eta)}
\end{align*}

Finally by taking $\epsilon = \frac{\lambda (1-\eta)^2}{16}$ and keeping only dominant term for each parameter, we obtain.
\begin{align*}
     \norm{\sum_{\tau = 1}^t \eta^{-\tau} \bphi_s \epsilon^{\w ,\A}_\tau }_{\tlSigma_t^{-1}}  
     \leq \mathcal{O}\left(d H  \sqrt{ \log \left( d H \beta \frac{1}{\lambda (1 - \eta)} \right) + \log (1/\delta)} \right)
\end{align*}

\end{proof}

\end{document}